\documentclass[preprint,12pt]{elsarticle}

\usepackage{amsmath,amssymb,here}
\usepackage{makecell,ntheorem,multirow}
\usepackage{thmtools,here}

\newcommand{\out}{\mbox{\scriptsize{out}}}
\newcommand{\nxt}{\mbox{\scriptsize{nxt}}}

\newtheorem{theorem}{Theorem}
\newtheorem{lemma}{Lemma}
\theorembodyfont{\upshape}
\newtheorem{example}{Example}
\newproof{proof}{Proof}

\begin{document}

\begin{frontmatter}

\title{Stronger Separation of Analog Neuron Hierarchy by Deterministic Context-Free Languages}

\author[ics]{Ji\v{r}\'{\i} \v{S}\'{\i}ma}
\ead{sima@cs.cas.cz}


\address[ics]{Institute of Computer Science of the Czech Academy of Sciences,\\
P.\,O.~Box~5, 18207~Prague~8, Czech Republic,
}

\begin{abstract}
We analyze the computational power of discrete-time recurrent neural networks (NNs) with the saturated-linear activation function within the Chomsky hierarchy. This model restricted to integer weights coincides with binary-state NNs with the Heaviside activation function, which are equivalent to finite automata (Chomsky level 3) recognizing regular languages (REG), while rational weights make this model Turing-complete even for three analog-state units (Chomsky level~0). For the intermediate model $\alpha$ANN of a~binary-state NN that is extended with $\alpha\geq 0$ extra analog-state neurons with rational weights, we have established the analog neuron hierarchy 0ANNs $\subset$ 1ANNs $\subset$ 2ANNs $\subseteq$ 3ANNs. The separation 1ANNs $\subsetneqq$ 2ANNs has been witnessed by the non-regular deterministic context-free language (DCFL) $L_\#=\{0^n1^n\,|\,n\geq 1\}$ which cannot be recognized by any 1ANN even with real weights, while any DCFL (Chomsky level~2) is accepted by a 2ANN with rational weights. In this paper, we strengthen this separation by showing that any non-regular DCFL cannot be recognized by 1ANNs with real weights, which means $($DCFLs $\setminus$ REG$)\,\subset\,($2ANNs $\setminus$ 1ANNs$)$, implying \mbox{1ANNs $\cap$ DCFLs $=$  0ANNs.} For this purpose, we have shown that $L_\#$ is the simplest non-regular DCFL by reducing $L_\#$ to any language in this class, which is by itself an interesting achievement in computability theory.
\end{abstract}




\end{frontmatter}

\section{Analog Neuron Hierarchy}

The standard techniques used in artificial neural networks (NNs) such as Hebbian learning, back-propagation, simulated annealing, support vector machines, deep learning, are of statistical or heuristic nature. NNs often considered as ``black box'' solutions are mainly subject to empirical research whose methodology is based on computer simulations through which the developed heuristics are tested, tuned, and mutually compared on benchmark data. Nevertheless, the development of NN methods has, among others, its own intrinsic limits given by mathematical, computability, or physical laws. By exploring these limits one can understand what is computable in principle or efficiently by NNs, which is a necessary prerequisite for pushing or even overcoming these boundaries in future intelligent technologies. Thus, rigorous mathematical foundations of NNs need to be further developed, which is the main motivation for this study. We explore the computational potential and limits of NNs for general-purpose computations by comparing them with more traditional computational models such as finite or pushdown automata, Chomsky grammars, and Turing machines.

The computational power of discrete-time recurrent NNs with the satur\-ated-linear activation function\footnote{The results are partially valid for more general classes of activation functions~\cite{Koiran96,Siegelmann96,Sima97,Sorel04} including the logistic function~\cite{Kilian96}.} depends on the descriptive complexity of their weight parameters~\cite{Siegelmann99,Sima03}. 
NNs with \emph{integer} weights, corresponding to binary-state (shortly binary) networks which employ the Heaviside activation function (with Boolean outputs 0 or~1), coincide with finite automata (FAs) recognizing regular languages~\cite{Alon91,Horne96,Indyk95,Minsky67,Sima14,Sima98}. \emph{Rational} weights make the analog-state (short\-ly analog) NNs (with real-valued outputs in the interval $[0,1]$) computationally equivalent to Turing machines (TMs)~\cite{Indyk95,Siegelmann95}, and thus (by the real-time simulation~\cite{Siegelmann95}) polynomial-time computations of such networks are characterized by the fundamental complexity class~P. 

In addition, NNs with arbitrary \emph{real} weights can even derive ``super-Turing'' computational capabilities~\cite{Siegelmann99}. Namely, their poly\-nomial-time computations correspond to the nonuniform complexity class P/poly while any input/output mapping (including algorithmically undecidable problems) can be computed within exponential time~\cite{Siegelmann94}. Moreover, a proper infinite hierarchy of nonuniform complexity classes between P and P/poly has been established for polynomial-time computations of NNs with increasing Kolmogorov complexity of real weights~\cite{Balcazar97}.

As can be seen, our understanding of the computational power of NNs is satisfactorily fine-grained when changing from rational to arbitrary real weights. In contrast, there is still a gap between integer and rational weights which results in a jump from regular languages (REG) capturing the lowest level 3 in the Chomsky hierarchy to recursively enumerable languages (RE) on the highest Chomsky level 0. In order to refine the classification of NNs which do not possess the full power of TMs (Chomsky level~0), we have initiated the study of binary-state NNs employing integer weights, that are extended with $\alpha\geq 0$ extra analog neurons having real weights, which are denoted as~$\alpha$ANNs.

This study has primarily been motivated by theoretical issues of how the computational power of NNs increases with enlarging analogicity when we change step by step from binary to analog states, or equivalently, from integer to arbitrary rational weights. In particular, the weights are mainly assumed to be just fixed fractions with a finite representation (i.e.\ a quotient of two integer constants) avoiding real numbers with infinite precision\footnote{Nevertheless, we formulate the present lower-bound results for arbitrary real weights which hold all the more so for rationals.}. Hence, the states of added $\alpha$ analog units can thus be only rationals although the number of digits in the representation of analog values may increase (linearly) along a computation. Nevertheless, by bounding the precision of analog states, we would reduce the computational power of NNs to that of finite automata which could be implemented by binary states. This would not allow the study of analogicity phenomena such as the transition from integer to rational weights in NNs whose functionality (program) is after all encoded in numerical weights.

There is nothing suspicious about the fact that the precision of analog states in $\alpha$ANNs is not limited by a fixed constant in advance. The same is true in conventional abstract models of computation such as pushdown automata or Turing machines with unlimited (potentially infinite) size of stack or tape, respectively, whose limitation would lead to the collapse of Chomsky hierarchy to finite automata. Thus, the proposed abstract model of $\alpha$ANNs itself has been intended for measuring the expressive power of a binary-state NN to which analog neurons are added one by one, rather than for solving special-purpose practical tasks or biological modeling. Nevertheless, as a secondary use, this analysis may potentially be relevant to practical hybrid NNs that combine binary and analog neurons in deep networks employing the LSTM, GRU or ReLU units~\cite{Schmidhuber15}, which deserves specialized studies such as in the recent work~\cite{Korsky19,Merrill19,Merrill20,Weiss18}.

In our previous work~\cite{Sima19}, we have characterized syntactically the class of languages that are accepted by 1ANNs with \emph{one} extra analog unit, in terms of so-called \emph{cut languages}\footnote{\label{dfcutl}A cut language $L_{<c}=\left\{x_1\ldots x_n\in A^*\,\left|\,\sum_{k=1}^nx_k\,\beta^{-k}<c\right.\right\}\subset A^*$ over a finite alphabet $A$ contains finite representations of numbers in a real \emph{base} $\beta$ (so-called \emph{$\beta$-expansions}) where $|\beta|>1$, using real \emph{digits} from~$A$, that are less than a given real \emph{threshold}~$c$. It is known that $L_{<c}$ is regular if{f} $c$ is quasi-periodic\footnotemark[4] while it is not context-free otherwise.}~\cite{Sima18}
which are combined in a certain way by usual operations such as complementation, intersection, union, concatenation, Kleene star, reversal, the largest prefix-closed subset, and a letter-to-letter morphism. By using this syntactic characterization of 1ANNs we have derived a sufficient condition when a 1ANN accepts only a regular language (Chomsky level 3), which is based on the \emph{quasi-periodicity}\footnote{\label{dfqper}For a real base $\beta$ satisfying $|\beta|>1$, and a finite alphabet $A$ of real digits, an infinite $\beta$-expansion\footnotemark[3], $\sum_{k=1}^{\infty}x_k\,\beta^{-k}$ where $x_k\in A$, is called \emph{quasi-periodic} if the sequence $\left(\sum_{k=1}^\infty x_{n+k}\,\beta^{-k}\right)_{n=0}^\infty$ contains a~constant infinite subsequence. We say that a real number $x$ is \emph{quasi-periodic} if all its infinite $\beta$-expansions $x=\sum_{k=1}^{\infty}x_k\,\beta^{-k}$ are quasi-periodic.}~\cite{Sima18} of some parameters depending on its real weights. This condition defines the subclass QP-1ANNs of so-called quasi-periodic 1ANNs which are computationally equivalent to FAs. For example, the class QP-1ANNs contains the 1ANNs with weights from the smallest field extension\footnote{\label{extQ}Recall that in algebra, the rational numbers (fractions) form the field $\mathbb{Q}$ with the two usual operations, the addition and the multiplication over real numbers. For any real number $\beta\in\mathbb{R}$, the \emph{field extension} $\mathbb{Q}(\beta)\subset\mathbb{R}$ is the smallest set containing $\mathbb{Q}\cup\{\beta\}$ that is closed under these operations. For example, the golden ratio $\varphi=(1+\sqrt{5})/2\in\mathbb{Q}(\sqrt{5})$ whereas $\sqrt{2}\notin\mathbb{Q}(\sqrt{5})$. Note that $\mathbb{Q}(\beta)=\mathbb{Q}$ for every $\beta\in\mathbb{Q}$.} $\mathbb{Q}(\beta)$ over the rational numbers $\mathbb{Q}$ including a Pisot number\footnote{\emph{Pisot number} is a real algebraic integer (a root of some monic polynomial with integer coefficients) greater than 1 such that all its Galois conjugates (other roots of such a unique monic polynomial with minimal degree) are in absolute value less than 1.} $\beta>1$, such that the self-loop weight $w$ of the only analog neuron equals $1/\beta$. For instance, the 1ANNs with arbitrary rational weights except for $w=1/n$ for some integer $n>1$, or $w=1/\varphi=1-\varphi$ where~$\varphi$ is the golden ratio, belong to QP-1ANNs recognizing REG. An example of the QP-1ANN ${\cal N}\!\left(27,\tfrac{1}{28}\right)$ that accepts the regular language (\ref{srctl}), is depicted in Figure~\ref{filrepthm} with parameters (\ref{exbetc2}).

On the other hand, we have introduced~\cite{Sima19} examples of languages accepted by 1ANNs with rational weights that are not context-free (CFLs) (i.e.\ are above Chomsky level~2), while we have proven that any language accepted by this model online\footnote{\label{online}In \emph{online} input/output protocols, the time between reading two consecutive input symbols as well as the delay in outputting the result after an input has been read, is bounded by a constant, while in \emph{offline} protocols these time  intervals are not bounded.}, is context-sensitive (CSL) at Chomsky level~1. For example, the 1ANN ${\cal N}\!\left(\tfrac{27}{8},\tfrac{1}{4}\right)$ depicted in Figure~\ref{filrepthm} with parameters (\ref{exbetc}), accepts the context-sensitive language 
\begin{equation}
\label{dfl1}
L_1={\cal L}\left({\cal N}\!\left(\tfrac{27}{8},\tfrac{1}{4}\right)\right)=\left\{x_1\ldots x_n\in\{0,1\}^*\,\left|\,\sum_{k=1}^n x_{n-k+1}\left(\tfrac{27}{8}\right)^{-k}<\tfrac{1}{4}\right.\right\}
\end{equation}
defined in (\ref{csctl}) as the reversal of a cut language, which is not context-free. These results refine the analysis of the computational power of NNs with the weight parameters between integer and rational weights. Namely, the computational power of binary-state networks having integer weights can increase from REG (Chomsky level 3) to that between CFLs (Chomsky level~2) and CSLs (Chomsky level 1), when an extra analog unit with rational weights is added, while the condition when adding one analog neuron even with real weights does not increase the power of binary-state networks, was formulated, which defines QP-1ANNs.

Furthermore, we have established the analog neuron hierarchy of classes of languages recognized by binary $\alpha$ANNs with $\alpha$ extra analog units having \emph{rational weights}, for $\alpha=0,1,2,3,\ldots$, that is, 0ANNs $\subseteq$ 1ANNs $\subseteq$ 2ANNs $\subseteq$ 3ANNs $\subseteq\cdots$, respectively~\cite{Sima20}. Note that we use the notation $\alpha$ANNs also for the class of languages accepted by $\alpha$ANNs, which can clearly be distinguished by the context. Obviously, the 0ANNs are purely binary-state NNs equivalent to FAs, which also implies 0ANNs~$=$~QP-1ANNs. Hence, 0ANNs $\subsetneqq$ 1ANNs because we know there are non-context-free languages such as $L_1$ in (\ref{dfl1}) accepted by 1ANNs~\cite{Sima19}. In contrast, we have proven that the non-regular deterministic context-free language (DCFL) 
\begin{equation}
\label{dfLh}
L_\#=\{0^n1^n\,|\,n\geq 1\}\,, 
\end{equation}
which contains the words of $n$ zeros followed by $n$ ones, cannot be recognized even offline\footnotemark[7] by any 1ANN with arbitrary real weights~\cite{Sima20}. We thus know that 1ANNs are not Turing-complete.

Nevertheless, we have shown that any DCFL included in Chomsky level~2 can be recognized by a 2ANN with two extra analog neurons having rational weights, by simulating a corresponding deterministic pushdown automaton (DPDA)~\cite{Sima20}. This provides the separation 1ANNs $\subsetneqq$ 2ANNs since the DCFL $L_\#$ in (\ref{dfLh}) is not accepted by any 1ANN. In addition, we have proven that any TM can be simulated by a 3ANN having rational weights with a linear-time overhead~\cite{Sima20}. It follows that RE at the highest Chomsky level~0 are accepted by 3ANNs with rational weights and thus this model including only three analog neurons is Turing-complete. Since $\alpha$ANNs with rational weights can be simulated by TMs for any $\alpha\geq 0$, the analog neuron hierarchy collapses to 3ANNs:
\begin{equation}
\mbox{FAs }\equiv\mbox{ 0ANNs }\subsetneqq\mbox{ 1ANNs }\subsetneqq\mbox{ 2ANNs }\subseteq\mbox{ 3ANNs }=\mbox{ 4ANNs }=\ldots
\equiv\mbox{ TMs}\nonumber
\end{equation}
It appears that the analog neuron hierarchy which is schematically depicted in Figure~\ref{analhier}, is only partially comparable to that of Chomsky.
\begin{figure}[t]
\centering
\includegraphics[width=13.7cm]{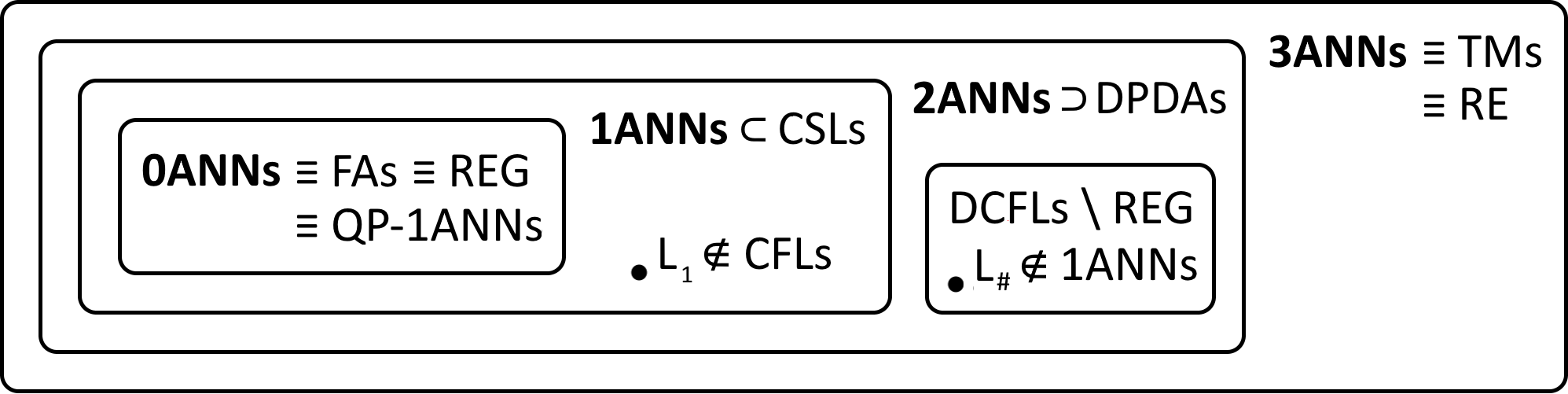}
\caption{The analog neuron hierarchy.}
\label{analhier}
\end{figure}

In this paper, we further study the relation between the analog neuron hierarchy and the Chomsky hierarchy. We show that any non-regular DCFL cannot be recognized online by 1ANNs with real weights, which provides the stronger separation
\begin{equation}
(\mbox{DCFLs }\setminus\mbox{ REG})\,\subset\,(\mbox{2ANNs } \setminus\mbox{ 1ANNs})\,,\nonumber
\end{equation}
implying REG $=$ 0ANNs $=$ QP-1ANNs $=$ 1ANNs $\cap$ DCFLs. Thus, the class of non-regular DCFLs is contained in 2ANNs with rational weights, having the empty intersection with 1ANNs, as depicted in Figure~\ref{analhier}.

In order to prove this lower bound on the computational power of 1ANNs, we have shown that the non-regular language $L_\#$ in (\ref{dfLh}) is in some sense the simplest DCFL (so-called DCFL-simple problem), by reducing $L_\#$ to any language in DCFLs $\setminus$ REG \cite{JanSim21}. Namely, given any non-regular DCFL $L$, we can recognize the language $L_\#$ by a Mealy machine (a deterministic finite-state transducer) that is allowed to call a~subroutine for deciding~$L$ (oracle) on its output extended with a few suffixes of constant length. In computability theory, this is a kind of truth-table (Turing) reduction by Mealy machines with an oracle for~$L$. In this paper, we prove that such a reduction can be implemented by an online 1ANN. Thus, if the non-regular DCFL $L$ were accepted by an online 1ANN, then we could recognize $L_\#$ by a 1ANN, which is a contradiction, implying that $L$ cannot be accepted by any online 1ANN even with real weights.

Note that the definition of DCFL-simple problems which any language in 
DCFLs $\setminus$ REG must include, is by itself an interesting achievement in formal language theory~\cite{JanSim21}. A DCFL-simple problem can be reduced to all the non-regular DCFL problems by the truth-table reduction using oracle Mealy machines, which is somewhat methodologically opposite to the usual hardness results in computational complexity theory where all problems in a class are reduced to its hardest problem such as in NP-completeness proofs. The concept of DCFL-simple problems has been motivated by our analysis of the computational power of 1ANNs and represents its first non-trivial application to proving the lower bounds. Our result can thus open a new direction of research in computability theory aiming towards the existence of the simplest problems in traditional complexity classes and their mutual reductions.

The paper is organized as follows. In Section~\ref{1ann}, we introduce basic definitions of the language acceptor based on 1ANNs, including an example of the 1ANNs that recognize the reversal of cut languages, which also illustrates its input/output protocol. In Section~\ref{techlem}, we prove two technical lemmas concerning the properties of 1ANNs which are used in Section~\ref{separation}
for the reduction of $L_\#$ to any non-regular DCFL by a~1ANN, implying that one extra analog neuron even with real weights is not sufficient for recognizing any non-regular DCFL online. Finally, we summarize the results and list some open problems in Section~\ref{concl}.

A preliminary version of this paper~\cite{Sima19c} contains only a sketch of the proof exploiting the representation of DCFLs by so-called deterministic monotonic restarting automata~\cite{Platek99}, while the complete argument for $L_\#$ to be the DCFL-simple problem has eventually been achieved by using DPAs \cite{JanSim21}.

\section{Neural Language Acceptors with One Analog Unit}
\label{1ann}

We specify the computational model of a \emph{discrete-time binary-state recurrent neural network with one extra analog unit} (shortly, 1ANN), ${\cal N}$, which will be used as a formal language acceptor. The network ${\cal N}$ consists of $s\geq 1$ \emph{units (neurons)}, indexed as $V=\{1,\ldots,s\}$. All the units in ${\cal N}$ are assumed to be binary-state (shortly \emph{binary}) neurons (i.e.\ perceptrons, threshold gates) except for the last $s$th neuron which is an analog-state (shortly \emph{analog}) unit. The neurons are connected into a directed graph representing an \emph{architecture} of ${\cal N}$, in which each edge $(i,j)\in V^2$ leading from unit $i$ to $j$ is labeled with a real \emph{weight} $w_{ji}\in\mathbb{R}$. The absence of a connection within the architecture corresponds to a zero weight between the respective neurons, and vice versa.

The computational dynamics of ${\cal N}$ determines for each unit $j\in V$ its \emph{state (output)} $y_j^{(t)}$ at discrete time instants $t=0,1,2,\ldots$. The states $y_j^{(t)}$ of the first $s-1$ binary neurons $j\in \tilde{V}=V\setminus\{s\}$ are Boolean values 0 or~1, whereas the output $y_s^{(t)}$ from the analog unit $s$ is a real number from the unit interval $\mathbb{I}=[0,1]$. This establishes the \emph{network state} $\mathbf{y}^{(t)}=\left(y_1^{(t)},\ldots,y_{s-1}^{(t)},y_s^{(t)}\right)\in\{0,1\}^{s-1}\times\mathbb{I}$ at each discrete time instant~$t\geq 0$. 

For notational simplicity, we assume the synchronous fully parallel mode without loss of efficiency~\cite{Orponen97}. At the beginning of a computation, the 1ANN~${\cal N}$ is placed in a predefined \emph{initial state} $\mathbf{y}^{(0)}\in\{0,1\}^{s-1}\times\mathbb{I}$. At discrete time instant $t\geq 0$, an \emph{excitation} of any neuron $j\in V$ is defined as 
\begin{equation}
\label{excitation}
\xi_j^{(t)}=\sum_{i=0}^s w_{ji}y_i^{(t)}\,,
\end{equation}
including a real \emph{bias} value $w_{j0}\in\mathbb{R}$ which, as usually, can be viewed as the weight from a formal constant unit input $y_0^{(t)}\equiv 1$ for every $t\geq 0$. At the next instant $t+1$, all the neurons $j\in V$ compute their new outputs $y_j^{(t+1)}$ in parallel by applying an \emph{activation function} $\sigma_j:\mathbb{R}\longrightarrow\mathbb{I}$ to $\xi_j^{(t)}$, that is,
\begin{equation}
\label{state}
y_j^{(t+1)}=\sigma_j\left(\xi_j^{(t)}\right)\quad\mbox{for every }j\in V\,.
\end{equation}
For the neurons $j\in \tilde{V}$ with binary states $y_j\in\{0,1\}$, the \emph{Heaviside} activation function $\sigma_j(\xi)=H(\xi)$ is used where
\begin{equation}
\label{heaviside}
H(\xi)=\left\{
\begin{array}{ll}
1&\quad\mbox{for }\xi\geq 0\\
0&\quad\mbox{for }\xi<0\,,
\end{array}
\right.
\end{equation}
while the analog unit $s\in V$ with real output $y_s\in\mathbb{I}$ employs the \emph{saturated-linear} function $\sigma_s(\xi)=\sigma(\xi)$ where 
\begin{equation}
\label{satlin}
\sigma(\xi)=\left\{
\begin{array}{ll}
1&\quad\mbox{for }\xi\geq 1\\
\xi&\quad\mbox{for }0<\xi<1\\
0&\quad\mbox{for }\xi\leq 0\,,
\end{array}
\right.
\end{equation}
In this way, the new network state $\mathbf{y}^{(t+1)}\in\{0,1\}^{s-1}\times\mathbb{I}$ is determined at time $t+1$.

The computational power of NNs has been studied analogously to the traditional models of computations~\cite{Sima03} so that the network is exploited as an acceptor of formal language $L\subseteq\Sigma^*$ over a finite alphabet~$\Sigma=\{\lambda_1,\ldots\lambda_q\}$ composed of $q$~letters (symbols). For the finite 1ANN ${\cal N}$, we use the following \emph{online} input/output protocol employing its special binary neurons $X\subset\tilde{V}$ and $\mbox{nxt},\mbox{out}\in \tilde{V}$. An input word (string) $\mathbf{x}=x_1\ldots x_n\in\Sigma^n$ of arbitrary length $n\geq 0$, is sequentially presented to the network, symbol after symbol, via the first $q<s$ so-called \emph{input neurons} $X=\{1,\ldots,q\}\subset \tilde{V}$, at the time instants $0<\tau_1<\tau_2<\cdots<\tau_n$ after queried by~${\cal N}$. The neuron $\mbox{nxt}\in \tilde{V}$ is used by ${\cal N}$ to prompt a user to enter the next input symbol. Thus, once the prefix $x_1,\ldots,x_{k-1}$ of~$\mathbf{x}$ for $1\leq k\leq n$, has been read, the next input symbol $x_k\in\Sigma$ is presented to~${\cal N}$ at the time instant $\tau_k$ that is one computational step after ${\cal N}$ activates the neuron $\mbox{nxt}\in \tilde{V}$. This means that ${\cal N}$ signals
\begin{equation}
\label{nxt}
y_{\nxt}^{(t-1)}=\left\{
\begin{array}{ll}
1&\mbox{if }t=\tau_k\\
0&\mbox{otherwise}
\end{array}
\right.
\enspace\mbox{for }k=1,\ldots,n\,.
\end{equation}
 
We employ the popular \emph{one-hot encoding} of alphabet $\Sigma$ where each letter $\lambda_i\in\Sigma$ is represented by one input neuron $i\in X$ which is activated when the symbol $\lambda_i$ is being read while, at the same time, the remaining input neurons $j\in X\setminus\{i\}$ do not fire. Namely, the states of input neurons $i\in X$, which represent a current input symbol $x_k\in\Sigma$ at the time instant $\tau_k$, are thus externally set as
\begin{equation}
\label{iprot}
y_i^{(t)}=\left\{
\begin{array}{ll}
1&\mbox{if }x_k=\lambda_i\mbox{ and }t=\tau_k\\
0&\mbox{otherwise}
\end{array}
\right.
\enspace\mbox{for }i\in X\mbox{ and }k=1,\ldots,n\,.
\end{equation}

At the same time, ${\cal N}$ carries out its computation deciding about each prefix of the input word $\mathbf{x}$ whether it belongs to~$L$, which is indicated by the output neuron $\mbox{out}\in \tilde{V}$ when the next input symbol is presented which is one step after the neuron nxt is active according to~(\ref{nxt}):
\begin{equation}
\label{oprot}
y_{\out}^{(\tau_{k+1})}=\left\{
\begin{array}{ll}
1&\enspace\mbox{if }x_1\ldots x_k\in L\\
0&\enspace\mbox{if }x_1\ldots x_k\notin L
\end{array}
\right.
\enspace\mbox{for }k=0,\ldots,n\,,
\end{equation}
where $\tau_{n+1}>\tau_n$ is the time instant when the input word $\mathbf{x}$ is decided (e.g.\ formally define $x_{n+1}$ to be any symbol from $\Sigma$ to ensure the consistency with the input protocol~(\ref{iprot}) for $k=n+1$). For instance, $y_{\out}^{(\tau_1)}=1$ if{f} the empty word $\varepsilon$ belongs to $L$. We assume the online protocol where $\tau_{k+1}-\tau_k\leq\delta$ for every $k=0,\ldots,n$ (formally $\tau_0=0$), is bounded by some integer constant $\delta>0$, which ensures ${\cal N}$ halts on every input word $\mathbf{x}\in\Sigma^*$. We say that a~language $L\subseteq\Sigma^*$ is \emph{accepted (recognized)} by 1ANN ${\cal N}$, which is denoted as $L={\cal L}({\cal N})$, if for any input word $\mathbf{x}\in\Sigma^*$, ${\cal N}$ accepts~$\mathbf{x}$ if{f} $\mathbf{x}\in L$.

\begin{example}
\label{runex}
{\rm
We illustrate the definition of the 1ANN language acceptor and its input/output protocol on a simple network ${\cal N}={\cal N}(\beta,c)$ with two real parameters $\beta>1$ and $c$. This 1ANN is used for recognizing a language ${\cal L}({\cal N})\subseteq\{0,1\}^*$ over the binary alphabet $\Sigma=\{\lambda_1,\lambda_2\}$ including $q=2$ binary digits $\lambda_1=0$ and $\lambda_2=1$. The network ${\cal N}$ is composed of $s=8$ neurons, that is, $V=\{1,\ldots,8\}$ where the last neuron $s=8\in V$ is the analog unit whereas $\tilde{V}=V\setminus\{8\}=\{1,\ldots,7\}$ contains the remaining binary neurons including the input neurons $X=\{1,2\}\subset\tilde{V}$ employing the one-hot encoding of the binary alphabet $\Sigma$,  and the neurons $\mbox{nxt}=3\in\tilde{V}$, $\mbox{out}=7\in\tilde{V}$ which implement the input/output protocol (\ref{nxt})--(\ref{oprot}).

The architecture of ${\cal N}(\beta,c)$ is depicted in Figure~\ref{filrepthm} where the directed edges connecting neurons are labeled with the respective weights $w_{82}=\beta^{-1}/\nu=(\beta-1)/\beta$, $w_{88}=\beta^{-1/3}$, $w_{4,\nxt}=w_{54}=w_{\nxt,5}=w_{65}=w_{68}=w_{\out,\nxt}=1$, and $w_{\out,6}=-1$, while the edges drawn without the originating formal unit $0$ correspond to the biases $w_{60}=-1-c/\nu=-1-(\beta-1)c$ and $w_{\nxt,0}=w_{40}=w_{50}=w_{\out,0}=-1$, where 
\begin{equation}
\label{dfnu}
\nu=\sum_{k=1}^\infty\beta^{-k}=\frac{1}{\beta-1}>0\,.
\end{equation}
We will first choose the parameters $\beta,c$ of ${\cal N}(\beta,c)$ so that the language ${\cal L}({\cal N}(\beta,c))$ is not CFL, while we will later reduce its power to a regular language for other parameters, that is,
\begin{equation}
\label{exbetc}
\beta=\left(\frac{3}{2}\right)^3=\frac{27}{8}>1 \qquad\mbox{and}\qquad c=\frac{1}{4} 
\end{equation}
which determine the parameterized weights and bias of ${\cal N}$,
\begin{equation}
w_{82}=\tfrac{19}{27}\,,\qquad
w_{88}=\tfrac{2}{3}\,,\qquad
w_{60}=-\,\tfrac{51}{32}\,.
\end{equation}
\begin{figure}[t]
\begin{center}
\includegraphics[width=13.7cm]{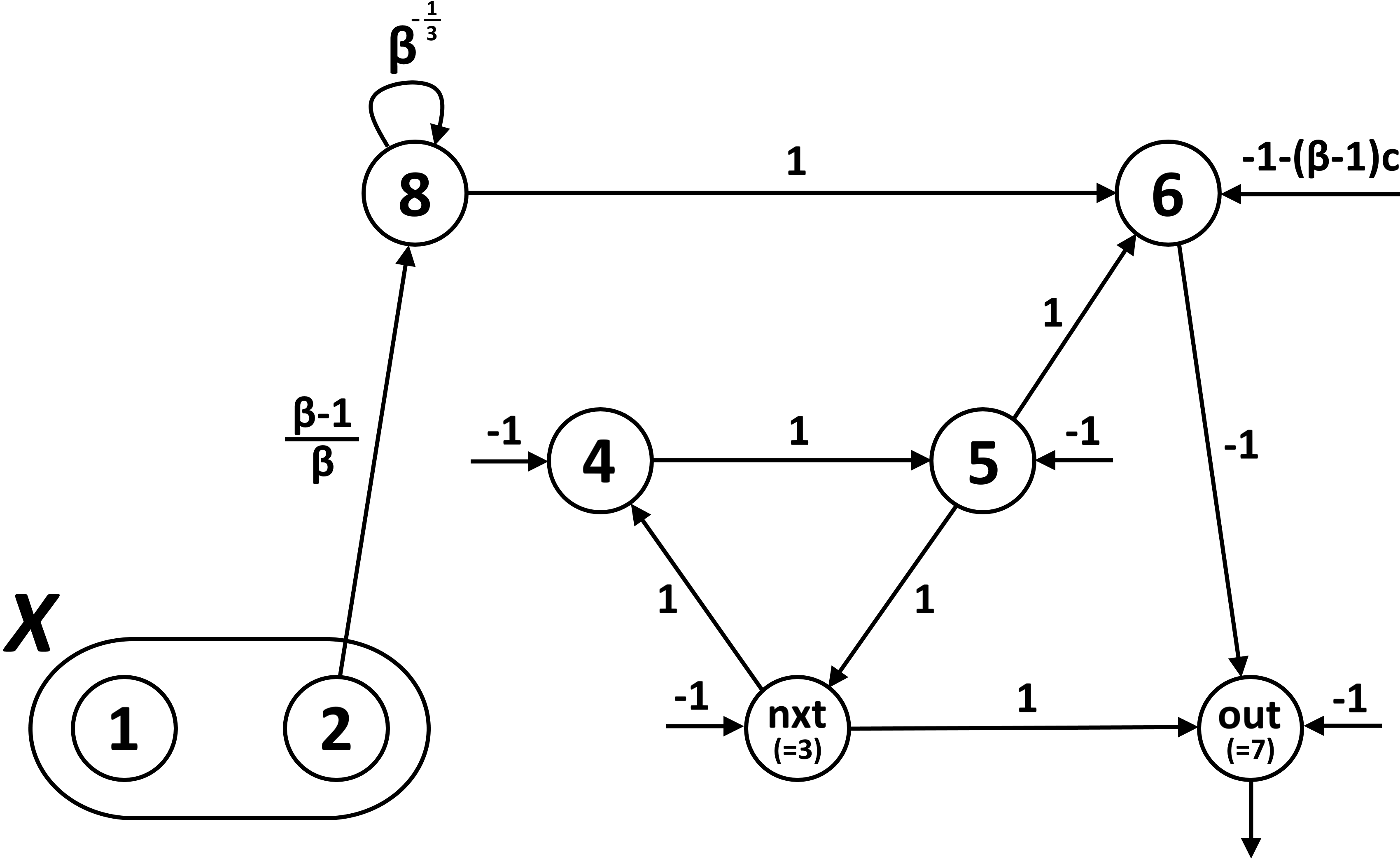}
\end{center}
\caption{The 1ANN language acceptor ${\cal N}(\beta,c)$.}
\label{filrepthm}
\end{figure}

\begin{table}[t]
\centering
\begin{tabular}{c||c|c|c|c|c|c|c|c||c}
\rule[-1.2ex]{0pt}{0pt}
$t$ & $y_1^{(t)}$ & $y_2^{(t)}$ & $y_{\nxt}^{(t)}$ & $y_4^{(t)}$ & $y_5^{(t)}$ & $y_6^{(t)}$  &  $y_{\out}^{(t)}$ & $y_8^{(t)}$ & \thead{the result of\\ recognition}
\\ 
\hline\hline\rule{0pt}{2.6ex}
0 & 0 & 0 & 1 & 0 & 0 & 0 & 0 & 0 & 
\\
\hline\rule[-1.2ex]{0pt}{0pt}\rule{0pt}{2.6ex}
1 & $\mathbf{0}$ & $\mathbf{1}$ & 0 & 1 & 0 & 0 & $\mathbf{1}$ & 0 & $\varepsilon\in {\cal L}({\cal N})$
\\
\hline\rule[-1.2ex]{0pt}{0pt}\rule{0pt}{2.6ex}
2 & 0 & 0 & 0 & 0 & 1 & 0 & 0 & $\tfrac{19}{27}$ &
\\
\hline\rule[-1.2ex]{0pt}{0pt}\rule{0pt}{2.6ex}
3 & 0 & 0 & 1 & 0 & 0 & 1 & 0 & $\tfrac{38}{81}$ &
\\
\hline\rule[-1.2ex]{0pt}{0pt}\rule{0pt}{2.6ex}
4 & $\mathbf{1}$ & $\mathbf{0}$ & 0 & 1 & 0 & 0 & $\mathbf{0}$ & $\tfrac{76}{243}$ & $1\notin {\cal L}({\cal N})$
\\
\hline\rule[-1.2ex]{0pt}{0pt}\rule{0pt}{2.6ex}
5 & 0 & 0 & 0 & 0 & 1 & 0 & 0 & $\tfrac{152}{729}$ &
\\
\hline\rule[-1.2ex]{0pt}{0pt}\rule{0pt}{2.6ex}
6 & 0 & 0 & 1 & 0 & 0 & 0 & 0 & $\tfrac{304}{2187}$ & 
\\
\hline\rule[-1.2ex]{0pt}{0pt}\rule{0pt}{2.6ex}
7 & $\mathbf{0}$ & $\mathbf{1}$ & 0 & 1 & 0 & 0 & $\mathbf{1}$ & $\tfrac{608}{6561}$ & $10\in{\cal L}({\cal N})$
\\
\hline\rule[-1.2ex]{0pt}{0pt}\rule{0pt}{2.6ex}
8 & 0 & 0 & 0 & 0 & 1 & 0 & 0 & $\tfrac{15067}{19683}$ &
\\
\hline\rule[-1.2ex]{0pt}{0pt}\rule{0pt}{2.6ex}
9 & 0 & 0 & 1 & 0 & 0 & 1 & 0 & $\tfrac{30134}{59049}$ &
\\
\hline\rule[-1.2ex]{0pt}{0pt}\rule{0pt}{2.6ex}
10 & $\mathbf{1}$ & $\mathbf{0}$ & 0 & 0 & 0 & 0 & $\mathbf{0}$ & $\tfrac{60268}{177147}$ & $101\notin{\cal L}({\cal N})$
\end{tabular}
\caption{\label{compn} The rejecting computation by the 1ANN ${\cal N}\!\left(\tfrac{27}{8},\tfrac{1}{4}\right)$ on the input $101$.}
\end{table}

Suppose that the input word $\mathbf{x}=101\in\{0,1\}^3$ of length $n=3$ is externally presented to ${\cal N}$ where $x_1=1$, $x_2=0$, $x_3=1$, and formally let $x_4=0$.
Table~\ref{compn} shows the sequential schedule of presenting the symbols $x_1$, $x_2$, $x_3$ of $\mathbf{x}$ to ${\cal N}$ through the input neurons $X=\{1,2\}\subset\tilde{V}$ at the time instants $\tau_1=1$, $\tau_2=4$, $\tau_3=7$, respectively, by using the one-hot coding, that is, $y_1^{(1)}=0$, $y_2^{(1)}=1$, $y_1^{(4)}=1$, $y_2^{(4)}=0$, $y_1^{(7)}=0$, $y_2^{(7)}=1$, according to (\ref{iprot}), which is indicated in boldface. Each input symbol is queried by the neuron $\mbox{nxt}\in\tilde{V}$ one step beforehand according to (\ref{nxt}). Thus, the neuron \mbox{nxt} is the only initially active unit, that is, $y_{\nxt}^{(0)}=1$, and this activity propagates repeatedly around the oriented cycle composed of three neurons $\mbox{nxt}(=3),4,5\in\tilde{V}$ through the edges with the unit weights, which ensures the neuron \mbox{nxt} fires only at the time instant $\tau_k-1=3(k-1)$ for $k>0$, when the next input symbol $x_k$ is prompted, whereas
\begin{equation}
\label{y53k11}
y_5^{(3k-1)}=1\quad\mbox{for every }k>0\,.
\end{equation}
In addition, the units $5$ and \mbox{nxt} from this cycle synchronize the incident neurons $6\in\tilde{V}$ and $\mbox{out}=7\in\tilde{V}$, respectively, so that the unit $6$ can be activated only at the time instants $t=3k$ for $k>0$, by (\ref{y53k11}), while the output neuron \mbox{out} can fire only at the time instants $\tau_{k+1}=3k+1$ for $k\geq 0$. Hence, the result of the recognition is reported by the output neuron \mbox{out} as indicated in Table~\ref{compn} in boldface, even for each of the four prefixes of $\mathbf{x}$, the empty string $\varepsilon$, $1$, $10$, and $101$, at the time instants $\tau_1=1$, $\tau_2=4$, $\tau_3=7$, $\tau_4=10$, respectively, according to (\ref{oprot}).

According to (\ref{excitation}), (\ref{state}), and (\ref{satlin}), we obtain the recurrence equation for the analog state of unit $8\in V$,
\begin{equation}
y_8^{(t)}=\xi_8^{(t-1)}=w_{82}\,y_2^{(t-1)}+w_{88}\, y_8^{(t-1)}=\frac{\beta^{-1}}{\nu}\,y_2^{(t-1)}+\beta^{-\frac{1}{3}}\,y_8^{(t-1)}
\end{equation}
at time instant $t\geq 1$, where $y_8^{(t)}=\xi_8^{(t-1)}\in\mathbb{I}$ by (\ref{dfnu}). Hence, the input symbols, which determine $y_2^{(3k+1)}=x_{k+1}$ at the time instants $\tau_{k+1}=3k+1$ for every $k\geq 0$, by the one-hot encoding, are stored in this analog state as
\begin{eqnarray}
y_8^{(1)}&=&y_8^{(0)}=0\\
y_8^{(2)}&=&\frac{\beta^{-1}}{\nu}\,x_1\\
y_8^{(4)}&=&\beta^{-\frac{1}{3}}\,y_8^{(3)}=\beta^{-\frac{2}{3}}\,y_8^{(2)}=
\frac{\beta^{-\frac{5}{3}}}{\nu}\,x_1\\
y_8^{(5)}&=&\frac{1}{\nu}\left(x_2\beta^{-1}+x_1\beta^{-2}\right)
\end{eqnarray}
etc., which generalizes to
\begin{equation}
\label{exandyn}
y_8^{(3k-1)}=\frac{1}{\nu}\,\sum_{i=1}^k x_{k-i+1}\,\beta^{-i}\,.
\end{equation}
It follows that the neuron $6\in\tilde{V}$, activating only at the time instant $t=3k$ for $k>0$, satisfies $y_6^{(3k)}=1$ if{f} $\xi_6^{(3k-1)}=w_{60}+w_{65}\,y_5^{(3k-1)}+w_{68}\,y_8^{(3k-1)}\geq 0$ if{f}
\begin{equation}
-1-\frac{c}{\nu}+1+\frac{1}{\nu}\,\sum_{i=1}^k x_{k-i+1}\,\beta^{-i}\geq 0
\end{equation}
according to (\ref{excitation})--(\ref{heaviside}), (\ref{y53k11}), and (\ref{exandyn}), which reduces to
\begin{equation}
y_6^{(3k)}=1\quad\mbox{if{f}}\quad\sum_{i=1}^k x_{k-i+1}\,\beta^{-i}\geq c\,.
\end{equation}
At the time instant $t=\tau_{k+1}=3k+1$, the output neuron $\mbox{out}\in\tilde{V}$ computes the negation of $y_6^{(3k)}$, and hence,
\begin{equation}
\label{exout}
y_{\out}^{(\tau_{k+1})}=1\quad\mbox{if{f}}\quad\sum_{i=1}^k x_{k-i+1}\,\beta^{-i}<c\,.
\end{equation}

It follows from (\ref{exout}) that the neural language acceptor ${\cal N}(\beta,c)$ accepts the reversal of the cut language\footnotemark[3],
\begin{equation}
\label{ctl}
{\cal L}({\cal N}(\beta,c))=L_{<c}^R=\left\{x_1\ldots x_n\in\{0,1\}^*\,\left|\,\sum_{k=1}^n x_{n-k+1}\,\beta^{-k}<c\right.\right\}\,.
\end{equation}
Since the threshold $c=\tfrac{1}{4}$ is not a quasi-periodic number\footnote{According to the definition of the quasi-periodic number\footnotemark[4], it suffices to prove that for some $\tfrac{27}{8}$-expansion $\sum_{k=1}^{\infty}x_{k}\left(\tfrac{27}{8}\right)^{-k}=\tfrac{1}{4}$ with $x_k\in\{0,1\}$, all the numbers $r_n=\sum_{k=n}^{\infty}x_{n+k}\left(\tfrac{27}{8}\right)^{-k}$ for $n\geq 0$, are distinct. Clearly, $r_0=\tfrac{1}{4}$ and $r_{n+1}=\tfrac{27}{8}r_n-x_{n+1}$ for every $n\geq 0$. One can show by induction on $n$ that $r_n=c_n/2^{3n+2}$ for some odd integer $c_n$, which provides the proof.} for the base $\beta=\tfrac{27}{8}$ and the binary digits $\{0,1\}$, the corresponding instance of (\ref{ctl}),
\begin{equation}
\label{csctl}
{\cal L}\left({\cal N}\!\left(\tfrac{27}{8},\tfrac{1}{4}\right)\right)=L_{<\tfrac{1}{4}}^R=\left\{x_1\ldots x_n\in\{0,1\}^*\,\left|\,\sum_{k=1}^n x_{n-k+1}\left(\tfrac{27}{8}\right)^{-k}<\tfrac{1}{4}\right.\right\}\,,
\end{equation}
is a context-sensitive language that is not \emph{context-free}~\cite{Sima18}. 

In contrast, if we choose the integer (Pisot) base and the quasi-periodic\footnotemark[4] threshold for this base,
\begin{equation}
\label{exbetc2}
\beta=3^3=27>1\qquad\mbox{and}\qquad c=\frac{1}{28}
\end{equation}
(cf.\ (\ref{exbetc})), respectively, for defining another instance of the 1ANN in Figure~\ref{filrepthm}, then the language accepted by this QP-1ANN ${\cal N}\!\left(27,\tfrac{1}{28}\right)$, which instantiates (\ref{ctl}) as
\begin{equation}
\label{rctl}
{\cal L}\left({\cal N}\!\left(27,\tfrac{1}{28}\right)\right)=L_{<\frac{1}{28}}^R=\left\{x_1\ldots x_n\in\{0,1\}^*\,\left|\,\sum_{k=1}^n x_{n-k+1}\,27^{-k}<\tfrac{1}{28}\right.\right\}\,,
\end{equation}
is \emph{regular}~\cite{Sima18}. The description of language~(\ref{rctl}) can be simplified as 
\begin{equation}
\label{srctl}
{\cal L}\left({\cal N}\!\left(27,\tfrac{1}{28}\right)\right)=\left\{x_1\ldots x_n\in\{0,1\}^*\,\left|\, x_n=0\right.\right\}\,,
\end{equation}
since for any $x_1\ldots x_{n-1}0\in\{0,1\}^*$, we have $\sum_{k=1}^n x_{n-k+1}\,27^{-k}<\sum_{k=2}^\infty 27^{-k}=\tfrac{1}{702}<\tfrac{1}{28}$, whereas $\sum_{k=1}^n x_{n-k+1}\,27^{-k}\geq\tfrac{1}{27}>\tfrac{1}{28}$ for every $x_1\ldots x_{n-1}1\in\{0,1\}^*$.
}
\end{example}

\section{Technical Properties of 1ANNs}
\label{techlem}

In this section, we will prove two lemmas about technical properties of 1ANNs that will be used in Section~\ref{separation} for implementing the reduction of $L_\#$ to any non-regular DCFL by a 1ANN. Namely, Lemma~\ref{partition} shows that for any time constant $T>0$, the state domain~$\mathbb{I}$ of the only analog unit of a 1ANN ${\cal N}$ can be partitioned into finitely many subintervals so that the binary states during $T$ consecutive computational steps by ${\cal N}$ are invariant to any initial analog state within each subinterval of this partition. Thus, one can extrapolate any computation by ${\cal N}$ for the next $T$ computational steps only on the basis of information to which subinterval the initial analog state belongs. Lemma~\ref{symdif} then shows that such an extrapolation can be evaluated by a binary neural network, which ensures that the class of 1ANNs is in fact closed under the (right) quotient with a word.\footnote{The (right) quotient of language $L$ with a word $\mathbf{u}$ is the language $L/\mathbf{u}=\{\mathbf{x}\mid\mathbf{x}\cdot\mathbf{u}\in L\}$.}
\begin{lemma}
\label{partition}
Let ${\cal N}$ be a 1ANN of size $s$ neurons, which can be exploited as an acceptor of languages over an alphabet $\Sigma$ for different initial states of ${\cal N}$. Then for every integer $T>0$, there exists a partition $I_1\cup I_2\cup\cdots\cup I_p=\mathbb{I}$ of the unit interval $\mathbb{I}=[0,1]$ into $p=O\left(s2^{sT}\right)$ intervals such that for any initial state $\mathbf{y}^{(0)}\in\{0,1\}^{s-1}\times\mathbb{I}$ and any input word $\mathbf{u}\in\Sigma^*$ of length $n=|\mathbf{u}|$ that meets $\tau_{n+1}\leq T$ according to the input/output protocol~(\ref{nxt})--(\ref{oprot}) for ${\cal N}$, the binary states $\tilde{\mathbf{y}}^{(t)}=\left(y_1^{(t)},\ldots,y_{s-1}^{(t)}\right)\in\{0,1\}^{s-1}$ at any time instant $t\in\{0,1,\ldots,\tau_{n+1}\}$, are uniquely determined only by the initial binary states $\tilde{\mathbf{y}}^{(0)}\in\{0,1\}^{s-1}$ and the index $r\in\{1,\ldots,p\}$ such that the initial state of the analog unit $s\in V$ satisfies $y_s^{(0)}\in I_r$. 
\end{lemma}

\begin{proof}
Let $T>0$ be an integer, $\mathbf{y}^{(0)}\in\{0,1\}^{s-1}\times\mathbb{I}$ be an initial state of ${\cal N}$, and $\mathbf{u}\in\Sigma^*$ of length $n=|\mathbf{u}|$ be an input word that meets $\tau_{n+1}\leq T$ according to the input/output protocol~(\ref{nxt})--(\ref{oprot}) for ${\cal N}$. Assume that 
\begin{equation}
\label{unsat}
0<\xi_s^{(t-1)}<1\quad\mbox{for every }t=1,\ldots,\tau-1 
\end{equation}
for some $\tau$ such that $0\leq\tau<\tau_{n+1}$, which implies 
$y_s^{(t)}=\xi_s^{(t-1)}$ for every $t=1,\ldots,\tau-1$, according to (\ref{state}) and (\ref{satlin}), and hence, for $\tau>0$,
\begin{eqnarray}
\label{xist0m1}
\xi_s^{(\tau-1)}&=&\sum_{i=0}^{s-1}w_{si}y_i^{(\tau-1)}+w_{ss}y_s^{(\tau-1)}\nonumber\\
&=&\sum_{i=0}^{s-1}w_{si}y_i^{(\tau-1)}+w_{ss}\left(\sum_{i=0}^{s-1}w_{si}y_i^{(\tau-2)}+w_{ss}y_s^{(\tau-2)}\right)\nonumber\\
\label{ystm}
\dots&=&\sum_{t=0}^{\tau-1}\left(\sum_{i=0}^{s-1}w_{si}y_i^{(t)}\right)w_{ss}^{\tau-t-1}+w_{ss}^\tau y_s^{(0)}\,.
\end{eqnarray}
Note that formula (\ref{xist0m1}) reduces to
\begin{equation}
\label{xiswss0}
\xi_s^{(\tau-1)}=\sum_{i=0}^{s-1}w_{si}y_i^{(\tau-1)}\,,
\end{equation}
when $w_{ss}=0$.

First assume $0<\xi_s^{(\tau-1)}<1$ when $\tau>0$, which implies
\begin{equation}
\label{unsatm}
y_s^{(\tau)}=\xi_s^{(\tau-1)}=\sum_{t=0}^{\tau-1}\left(\sum_{i=0}^{s-1}w_{si}y_i^{(t)}\right)w_{ss}^{\tau-t-1}+w_{ss}^\tau y_s^{(0)}
\end{equation}
according to (\ref{state}), (\ref{satlin}), and (\ref{ystm}). For any binary neuron $j\in \tilde{V}$, we have
\begin{equation}
\label{yj1}
y_j^{(\tau+1)}=1\quad\mbox{if{f}}\quad\xi_j^{(\tau)}=\sum_{i=0}^{s-1}w_{ji}y_i^{(\tau)}+w_{js}y_s^{(\tau)}\geq 0
\end{equation}
according to (\ref{state}) and (\ref{heaviside}). By plugging (\ref{unsatm}) into (\ref{yj1}), we obtain
\begin{equation}
\label{yj1p}
y_j^{(\tau+1)}=1\quad\mbox{if{f}}\quad
\sum_{i=0}^{s-1}w_{ji}y_i^{(\tau)}+w_{js}\sum_{t=0}^{\tau-1}\left(\sum_{i=0}^{s-1}w_{si}y_i^{(t)}\right)w_{ss}^{\tau-t-1}+w_{js}w_{ss}^\tau y_s^{(0)}\geq 0\,,\,
\end{equation}
which can be rewritten for $w_{ss}\not=0$ and $w_{js}\not=0$ as
\begin{equation}
y_j^{(\tau+1)}=1\quad\mbox{if{f}}\quad
\end{equation}
\begin{equation}
\label{yj1p2}
\sum_{t=0}^{\tau-1}\left(-\sum_{i=0}^{s-1}\frac{w_{si}}{w_{ss}}y_i^{(t)}\right)w_{ss}^{-t}
-\sum_{i=0}^{s-1}\frac{w_{ji}}{w_{js}}y_i^{(\tau)}w_{ss}^{-\tau}
\left\{
\begin{array}{ll}
\geq y_s^{(0)}&\mbox{if }w_{js}w_{ss}^\tau<0\\
\leq y_s^{(0)}&\mbox{if }w_{js}w_{ss}^\tau>0\,.
\end{array}
\right.
\end{equation}
For $w_{ss}=0$ and $\tau>0$, condition (\ref{yj1p}) reduces to
\begin{equation}
\label{wss0}
y_j^{(\tau+1)}=1\quad\mbox{if{f}}\quad
\sum_{i=0}^{s-1}w_{ji}y_i^{(\tau)}+w_{js}\left(\sum_{i=0}^{s-1}w_{si}y_i^{(\tau-1)}\right)\geq 0
\end{equation}
which means the state $y_j^{(\tau+1)}$ depends in fact only on the binary states $\tilde{\mathbf{y}}^{(\tau)}$ and $\tilde{\mathbf{y}}^{(\tau-1)}$ where $\tilde{\mathbf{y}}^{(t)}=\left(y_1^{(t)},\ldots,y_{s-1}^{(t)}\right)\in\{0,1\}^{s-1}$. Similarly, for $w_{js}=0$, we have
\begin{equation}
\label{wjs0}
y_j^{(\tau+1)}=1\quad\mbox{if{f}}\quad
\sum_{i=0}^{s-1}w_{ji}y_i^{(\tau)}\geq 0
\end{equation}
when the state $y_j^{(\tau+1)}$ depends only on the binary states $\tilde{\mathbf{y}}^{(\tau)}$.

For the case when either $\xi_s^{(\tau-1)}\leq 0$ or $\xi_s^{(\tau-1)}\geq 1$ for $w_{ss}\not=0$ and $\tau>0$, we have 
\begin{eqnarray}
\label{xist0m10}
y_s^{(\tau)}=0&\mbox{if{f}}&\sum_{t=0}^{\tau-1}\left(-\sum_{i=0}^{s-1}\frac{w_{si}}{w_{ss}}y_i^{(t)}\right)w_{ss}^{-t}\,
\left\{
\begin{array}{ll}
\geq y_s^{(0)}&\mbox{if }w_{ss}^\tau>0\\
\leq y_s^{(0)}&\mbox{if }w_{ss}^\tau<0
\end{array}
\right.\\
\label{xist0m11}
y_s^{(\tau)}=1&\mbox{if{f}}&\frac{1}{w_{ss}^\tau}+\sum_{t=0}^{\tau-1}\left(-\sum_{i=0}^{s-1}\frac{w_{si}}{w_{ss}}y_i^{(t)}\right)w_{ss}^{-t}\,
\left\{
\begin{array}{ll}
\geq y_s^{(0)}&\mbox{if }w_{ss}^\tau<0\\
\leq y_s^{(0)}&\mbox{if }w_{ss}^\tau>0\,,
\end{array}
\right.\quad
\end{eqnarray}
respectively, according to (\ref{state}), (\ref{satlin}), and (\ref{xist0m1}).

Altogether, for any $\ell\in V$ such that $w_{\ell s}\not=0$, and $\tilde{\mathbf{y}}=\left(y_1,\ldots,y_{s-1}\right)\in\{0,1\}^{s-1}$, we denote 
\begin{equation}
\label{dfzeta}
\zeta_\ell\left(\tilde{\mathbf{y}}\right)=-\sum_{i=0}^{s-1}\frac{w_{\ell i}}{w_{\ell s}}y_i\,,
\end{equation}
which reduces conditions (\ref{yj1p2}), (\ref{xist0m10}), (\ref{xist0m11}) with $w_{ss}\not=0$ to
\begin{eqnarray}
\label{yj1p3}
y_j^{(\tau+1)}=1\,&\mbox{if{f}}&z_j\left(\tilde{\mathbf{y}}^{(0)},\tilde{\mathbf{y}}^{(1)},
\ldots,\tilde{\mathbf{y}}^{(\tau)}\right)
\left\{
\begin{array}{ll}
\geq y_s^{(0)}&\mbox{if }w_{js}w_{ss}^\tau<0\\
\leq y_s^{(0)}&\mbox{if }w_{js}w_{ss}^\tau>0
\end{array}
\right.\quad
\end{eqnarray}
for $j\in \tilde{V}$ such that $w_{js}\not=0$,
\begin{eqnarray}
\label{yst0m0}
y_s^{(\tau)}=0\,&\mbox{if{f}}&z_s\left(\tilde{\mathbf{y}}^{(0)},\tilde{\mathbf{y}}^{(1)},\ldots,\tilde{\mathbf{y}}^{(\tau-1)}\right)
\left\{
\begin{array}{ll}
\geq y_s^{(0)}&\mbox{if }w_{ss}^\tau>0\\
\leq y_s^{(0)}&\mbox{if }w_{ss}^\tau<0
\end{array}
\right.\\
\label{yst0m1}
y_s^{(\tau)}=1\,&\mbox{if{f}}&\frac{1}{w_{ss}^\tau}+z_s\left(\tilde{\mathbf{y}}^{(0)},\tilde{\mathbf{y}}^{(1)},
\ldots,\tilde{\mathbf{y}}^{(\tau-1)}\right)
\left\{
\begin{array}{ll}
\geq y_s^{(0)}&\mbox{if }w_{ss}^\tau<0\\
\leq y_s^{(0)}&\mbox{if }w_{ss}^\tau>0\,,
\end{array}
\right.
\end{eqnarray}
for $\tau>0$, respectively, where
\begin{equation}
\label{dfzel}
z_\ell\left(\tilde{\mathbf{y}}_0,
\tilde{\mathbf{y}}_1,\ldots,\tilde{\mathbf{y}}_\tau\right)=
\sum_{t=0}^{\tau-1}\zeta_s\left(\tilde{\mathbf{y}}_t\right)w_{ss}^{-t}
+\zeta_\ell\left(\tilde{\mathbf{y}}_m\right)w_{ss}^{-\tau}\,.
\end{equation}

We define the set
\begin{eqnarray}
Z&=&\big(Z'\cap(\mathbb{I}\times\{-1,1\})\big)\cup\big\{(0,-1),(0,1),(1,-1),(1,1)\big\}\nonumber\\
\label{dfZ}
&=&\big\{(a_1,b_1),(a_2,b_2),\ldots,(a_{p+1},b_{p+1})\big\}
\,\subset\,\mathbb{I}\times\{-1,1\}
\end{eqnarray}
where
\begin{eqnarray}
Z'&=&\left\{\big(z_j\left(\tilde{\mathbf{y}}_0,
\ldots,\tilde{\mathbf{y}}_\tau\right),-\mbox{sgn}\left(w_{js}w_{ss}^\tau\right)\big)
\left|
\begin{array}{c}
j\in \tilde{V}\,\mbox{ s.t. }\,w_{js}\not=0\\
\tilde{\mathbf{y}}_0\ldots,\tilde{\mathbf{y}}_\tau\in\{0,1\}^{s-1}\\
\enspace 0\leq\tau<T
\end{array}
\right.\!\!\!
\right\}\nonumber\\
&\bigcup&
\left\{\big(z_s\left(\tilde{\mathbf{y}}_0,
\ldots,\tilde{\mathbf{y}}_{\tau-1}\right),\mbox{sgn}\left(w_{ss}^\tau\right)\big)
\left|
\begin{array}{c}
\tilde{\mathbf{y}}_0\ldots,\tilde{\mathbf{y}}_{\tau-1}\in\{0,1\}^{s-1}\\
\enspace 0<\tau<T
\end{array}
\right.\!\!\!
\right\}\\
\label{dfZp}
&\bigcup&
\left\{\left(\frac{1}{w_{ss}}+z_s\left(\tilde{\mathbf{y}}_0,
\ldots,\tilde{\mathbf{y}}_{\tau-1}\right),-\mbox{sgn}\left(w_{ss}^\tau\right)\right)
\left|
\begin{array}{c}
\tilde{\mathbf{y}}_0\ldots,\tilde{\mathbf{y}}_{\tau-1}\in\{0,1\}^{s-1}\\
\enspace 0<\tau<T
\end{array}
\right.\!\!\!
\right\}\qquad\nonumber
\end{eqnarray}
and $\mbox{sgn}:\mathbb{R}\rightarrow\{-1,0,1\}$ is the signum function.
The set $Z$ includes the $p+1$ pairs $(a_r,b_r)\in\mathbb{I}\times\{-1,1\}$
for $r=1,\ldots,p+1$, which encode all the possible closed half-lines with the finite endpoints $a_r\in\mathbb{I}=[0,1]$, either $[a_r,+\infty)$ if $b_r=-1$, or $(-\infty,a_r]$ if $b_r=1$, that may occur in conditions (\ref{yj1p3})--(\ref{yst0m1}) determining the binary outputs $y_j^{(\tau+1)},y_s^{(\tau)}\in\{0,1\}$ for the analog state $y_s^{(0)}\in\mathbb{I}$. Clearly, the number $|Z|=p+1$ of these half-lines can be bounded as
\begin{eqnarray}
p+1&\leq&(s-1)\left( 2^{s-1}+\left(2^{s-1}\right)^2+\cdots+\left(2^{s-1}\right)^T\right)\nonumber\\
&&+2\left(\left(2^{s-1}\right)^2+\cdots+\left(2^{s-1}\right)^{T-1}\right)+4=O\left(s2^{sT}\right)\,.
\end{eqnarray}
We also assume that the elements of $Z$ are lexicographically sorted as
\begin{equation}
\label{lexico}
(a_1,b_1)<(a_2,b_2)<\cdots<(a_{p+1},b_{p+1})
\end{equation}
which is used in the definition of the partition of the unit interval $\mathbb{I}=[0,1]=I_1\cup I_2\cup\ldots\cup I_p$ into $p$ intervals:
\begin{equation}
\label{dfir}
I_r=\left\{
\begin{array}{ll}
[a_r,a_{r+1})&\,\mbox{if}\enspace b_r=-1\enspace\&\enspace b_{r+1}=-1\\
{[}a_r,a_{r+1}]&\,\mbox{if}\enspace b_r=-1\enspace\&\enspace b_{r+1}=1\\
(a_r,a_{r+1})&\,\mbox{if}\enspace b_r=1\enspace\&\enspace b_{r+1}=-1\\
(a_r,a_{r+1}]&\,\mbox{if}\enspace b_r=1\enspace\&\enspace b_{r+1}=1
\end{array}
\right.\quad\mbox{for }\,r=1,\ldots,p\,.
\end{equation}
Note that if $a_r=a_{r+1}$ for some $r\in\{1,\ldots,p\}$, then we know $-1=b_r<b_{r+1}=1$ due to $Z$ is lexicographically sorted, which produces the degenerate interval $I_r=[a_r,a_r]$. Thus, $I_1=[0,0]$ and $I_p=[1,1]$ because 
$(0,-1),(0,1),(1,-1),(1,1)\in Z$ according to (\ref{dfZ}). 

We will show that for any initial binary states $\tilde{\mathbf{y}}^{(0)}\in\{0,1\}^{s-1}$, the binary output $y_j^{(\tau)}\in\{0,1\}$ from any neuron $j\in \tilde{V}$ after the next $\tau$ computational steps of ${\cal N}$ where $0\leq \tau\leq\tau_{n+1}\leq T$, is the same for all initial analog values $y_s^{(0)}$ within the whole interval $I_r$, which means $\tilde{\mathbf{y}}^{(\tau)}$ depends only on $\tilde{\mathbf{y}}^{(0)}$ and $r\in\{1,\ldots,p\}$ such that $y_s^{(0)}\in I_r$. We proceed by induction on $\tau=0,\ldots,\tau_{n+1}$ satisfying~(\ref{unsat}). The base case is trivial since $\tilde{\mathbf{y}}^{(0)}$ does not depend on $y_s^{(0)}$ at all. Thus assume in the induction step that the statement holds for $\tilde{\mathbf{y}}^{(0)},\tilde{\mathbf{y}}^{(1)},\ldots,\tilde{\mathbf{y}}^{(\tau)}$ that meet~(\ref{unsat}), where $0\leq \tau<\tau_{n+1}$. 

Consider first the case when either $\tau=0$ or $0<\xi_s^{(\tau-1)}<1$ for $\tau>0$ which ensures (\ref{unsatm}) and extends the validity of condition~(\ref{unsat}) for $\tau$ replaced by $\tau+1$ in the next inductive step. Further assume $w_{ss}\not=0$ and let $j\in \tilde{V}$ be any binary neuron. For $w_{js}=0$, the state $y_j^{(\tau+1)}$ is clearly determined only by $\tilde{\mathbf{y}}^{(\tau)}$ according to (\ref{wjs0}). For $w_{js}\not=0$, the binary state $y_j^{(\tau+1)}\in\{0,1\}$ depends on whether the initial analog output $y_s^{(0)}\in\mathbb{I}$ lies on the corresponding half-line from $Z$ with the endpoint $z_j(\tilde{\mathbf{y}}^{(0)},\tilde{\mathbf{y}}^{(1)},\ldots,\tilde{\mathbf{y}}^{(\tau)})$, according to (\ref{yj1p3}), which holds within the whole interval $I_r\ni y_s^{(0)}$, since the endpoints $z_j\left(\tilde{\mathbf{y}}_0,\tilde{\mathbf{y}}_1,\ldots,\tilde{\mathbf{y}}_\tau\right)$ of all the possible half-lines in condition (\ref{yj1p3}) for $0\leq\tau<T$, are taken into account in the definition (\ref{dfZ}), (\ref{dfZp}) determining the partition (\ref{dfir}) of the analog state domain $\mathbb{I}$. Thus, $\tilde{\mathbf{y}}^{(\tau+1)}$ depends only on $\tilde{\mathbf{y}}^{(\tau)}$ and $I_r$ containing $y_s^{(0)}$, and hence, only on $\tilde{\mathbf{y}}^{(0)}$ and $r\in\{1,\ldots,p\}$ such that $y_s^{(0)}\in I_r$, by induction hypothesis. For $w_{ss}=0$, we know that $\tilde{\mathbf{y}}^{(\tau+1)}$ depends only on  $\tilde{\mathbf{y}}^{(\tau)}$ and $\tilde{\mathbf{y}}^{(\tau-1)}$ according to~(\ref{wss0}), which proves the assertion for $\tau>0$ by induction hypothesis, while for $\tau=0$ the argument is the same as for $w_{ss}\not=0$ since condition (\ref{yj1p3}) makes still sense for $\tau=0$. This completes the induction step for $\tau=0$ or $0<\xi_s^{(\tau-1)}<1$ for $\tau>0$. 

In the case when either $\xi_s^{(\tau-1)}\leq 0$ or $\xi_s^{(\tau-1)}\geq 1$ for $\tau>0$, we know the analog output $y_s^{(\tau)}\in\{0,1\}$ is, in fact, binary, satisfying (\ref{yst0m0}) or (\ref{yst0m1}) when $w_{ss}\not=0$, respectively, which means $y_s^{(0)}\in\mathbb{I}$ lies on the corresponding half-line from $Z$ with the endpoint $z_s(\tilde{\mathbf{y}}^{(0)},\tilde{\mathbf{y}}^{(1)},\ldots,\tilde{\mathbf{y}}^{(\tau-1)})$. This holds within the whole interval $I_r\ni y_s^{(0)}$, since the endpoints $z_s\left(\tilde{\mathbf{y}}_0,\tilde{\mathbf{y}}_1,\ldots,\tilde{\mathbf{y}}_{\tau-1}\right)$ of all the possible half-lines in conditions (\ref{yst0m0}) and (\ref{yst0m1}) for $0<\tau<T$, are taken into account in the definition (\ref{dfZ}), (\ref{dfZp}) determining the partition (\ref{dfir}).
For $w_{ss}=0$, the state $y_s^{(\tau)}\in\{0,1\}$ depends only on  $\tilde{\mathbf{y}}^{(\tau-1)}$ according to (\ref{xiswss0}).
Thus, $\tilde{\mathbf{y}}^{(\tau+1)}$ is determined by the binary state $\mathbf{y}^{(\tau)}\in\{0,1\}^s$ that is guaranteed for the whole interval $I_r$ containing $y_s^{(0)}$, and hence, $\tilde{\mathbf{y}}^{(\tau+1)}$ depends only on $\tilde{\mathbf{y}}^{(0)}$ and $r\in\{1,\ldots,p\}$ such that $y_s^{(0)}\in I_r$, by induction hypothesis. In addition, the same holds for the subsequent binary states $\tilde{\mathbf{y}}^{(\tau+2)},\tilde{\mathbf{y}}^{(\tau+3)},\ldots,\tilde{\mathbf{y}}^{(\tau_{n+1})}$ which are also determined by the binary state $\mathbf{y}^{(\tau)}\in\{0,1\}^s$ at the time instant $\tau$, which completes the proof of Lemma~\ref{partition}.~\qed 
\end{proof}

\begin{lemma}
\label{symdif}
Let ${\cal N}$ be a 1ANN which recognizes the language $L={\cal L}({\cal N})\subseteq\Sigma^*$ over an alphabet $\Sigma$ by using the online input/output protocol (\ref{nxt})--(\ref{oprot}) satisfying $\tau_{k+1}-\tau_{k}\leq\delta$ for every $k\geq 0$ and some integer constant $\delta>0$. Let $\mathbf{u}_1,\mathbf{u}_2\in\Sigma^+$ be two nonempty strings which define the (right) quotients $L_1=L/\mathbf{u}_1$ and $L_2=L/(\mathbf{u}_2\cdot\mathbf{u}_1)$ of $L$ with $\mathbf{u}_1$ and $\mathbf{u}_2\cdot\mathbf{u}_1$, respectively, where $L/\mathbf{u}=\{\mathbf{x}\in\Sigma^*\mid \mathbf{x}\cdot\mathbf{u}\in L\}$. Then there exists a 1ANN ${\cal N}'$ that accepts ${\cal L}({\cal N}')=L_2\setminus L_1$ respectively ${\cal L}({\cal N}')=L_1\setminus L_2$,
with the delay of 3~computational steps, that is, the output protocol (\ref{oprot}) is modified for ${\cal N}'$ as $y_{\out'}^{(\tau_{k+1}+3)}=1$ if{f} $x_1\ldots x_k\in {\cal L}({\cal N}')$, where $\mbox{out}'\in\tilde{V}'$ is the binary output neuron of ${\cal N}'$. 
\end{lemma}

\begin{proof}
We will construct the 1ANN ${\cal N}'$ such that ${\cal L}({\cal N}')=L_2\setminus L_1$ respectively ${\cal L}({\cal N}')=L_1\setminus L_2$ for the delayed output protocol, which contains ${\cal N}$ with $s$ neurons as its subnetwork including the analog unit $s\in V$ shared by ${\cal N}'$, that is, $V\subset V'=\tilde{V}'\cup\{s\}$ for the corresponding sets of (binary) neurons. The architecture of ${\cal N}'$ is schematically depicted in Figure~\ref{quotient}. Let $I_1\cup I_2\cup\cdots\cup I_p=\mathbb{I}$ be the partition of the state domain $\mathbb{I}=[0,1]$ of the analog unit $s\in V$ in ${\cal N}$ into $p$ intervals according to Lemma~\ref{partition} for $T=\delta\cdot(|\mathbf{u}_2\mathbf{u}_1|+1)$. We encode these intervals by the $p+1$ pairs $(a_r,b_r)\in\mathbb{I}\times\{-1,1\}$ for $r=1,\ldots,p+1$, according to (\ref{dfir}) where $a_r\in\mathbb{I}$ is the left endpoint of $I_r$ and $b_r=1$ if $I_r$ is left-open, while $b_r=-1$ if $I_r$ is left-closed, which are lexicographically sorted according to (\ref{lexico}).

For each pair $(a_r,b_r)$ where $r\in\{1,\ldots,p+1\}$, we introduce one binary neuron $\alpha_r\in \tilde{V}'$ in ${\cal N}'$ to which the analog unit $s\in V$ is connected so that $y_{\alpha_r}^{(t_0+1)}=1$ if{f}
\begin{equation}
y_s^{(t_0)}\left\{
\begin{array}{ll}
\geq a_r&\mbox{for }\,b_r=-1\\
\leq a_r&\mbox{for }\,b_r=1
\end{array}
\right.
\end{equation}
if{f} $b_ra_r-b_ry_s^{(t_0)}\geq 0$, for any time instant $t_0\geq 0$. According to (\ref{excitation})--(\ref{heaviside}), the bias and the corresponding weight of $\alpha_r\in \tilde{V}'$ from $s$ are thus defined as $w'_{\alpha_r,0}=b_ra_r$ and $w'_{\alpha_r,s}=-b_r$, respectively (see Figure~\ref{quotient}). Clearly, the binary states $\mathbf{y}_\alpha^{(t_0+1)}=\left(y_{\alpha_1}^{(t_0+1)},\ldots,y_{\alpha_{p+1}}^{(t_0+1)}\right)\in\{0,1\}^{p+1}$ of neurons in $\alpha=\{\alpha_1,\ldots,\alpha_{p+1}\}\subset\tilde{V}'$ at time $t_0+1$ determine uniquely the index $r\in\{1,\ldots,p+1\}$ such that $y_s^{(t_0)}\in I_r$. 
In addition, for the synchronization purpose, we introduce the set $\beta=\{\beta_1,\ldots,\beta_{s-1}\}\subset\tilde{V}'$ of $s-1$ binary neurons in ${\cal N}'$ that, at the time instant $t_0+1$, copy the binary states $\tilde{\mathbf{y}}^{(t_0)}=\left(y_1^{(t_0)},\ldots,y_{s-1}^{(t_0)}\right)\in\{0,1\}^{s-1}$ of ${\cal N}$ from the time instant $t_0$, which means $\mathbf{y}_\beta^{(t_0+1)}=\left(y_{\beta_1}^{(t_0+1)},\ldots,y_{\beta_{s-1}}^{(t_0+1)}\right)=\tilde{\mathbf{y}}^{(t_0)}$. This can implemented by the biases \mbox{$w'_{\beta_i,0}=-1$} and weights $w'_{\beta_i,i}=1$ for every $i=1,\ldots,s-1$, according to (\ref{excitation})--(\ref{heaviside}).
\begin{figure}[h]
\centering
\includegraphics[width=13.7cm]{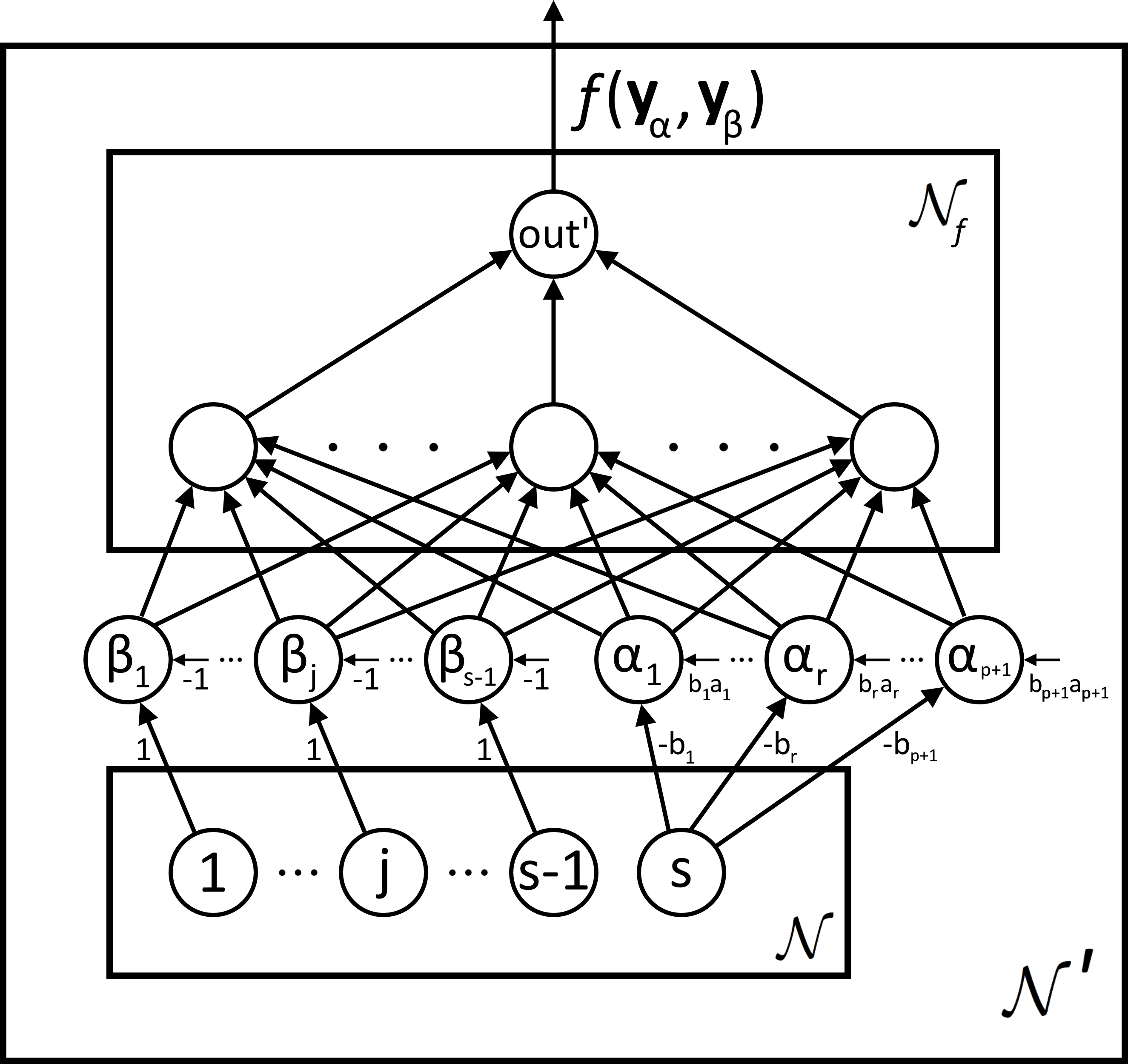}
\caption{The 1ANN ${\cal N}'$ that, with the delay of 3 steps, accepts ${\cal L}({\cal N}')=L_2\setminus L_1$ respectively ${\cal L}({\cal N}')=L_1\setminus L_2$, where $L_1={\cal L}({\cal N})/\mathbf{u}_1$ and $L_2={\cal L}({\cal N})/(\mathbf{u}_2\cdot\mathbf{u}_1)$.}
\label{quotient}
\end{figure}

For any input word $\mathbf{x}\in\Sigma^*$ of length $n=|\mathbf{x}|$, let $t_0\geq 0$ be a time instant when $\mathbf{x}$ has been read and still not decided by ${\cal N}$, that is, $\tau_{n}\leq t_0<\tau_{n+1}$ according to the input protocol (\ref{nxt})--(\ref{iprot}). According to Lemma~\ref{partition}, for the state $\mathbf{y}^{(t_0)}\in\{0,1\}^{s-1}\times\mathbb{I}$ that is considered as an initial state of ${\cal N}$ and for any nonempty suffix string $\mathbf{u}\in\Sigma^+$ added to $\mathbf{x}$ such that $\delta(|\mathbf{u}|+1)\leq T$, which is presented to ${\cal N}$ as an input since the time instant~$t_0$, the binary states $\tilde{\mathbf{y}}^{(t_0+\tau)}=\left(y_1^{(t_0+\tau)},\ldots,y_{s-1}^{(t_0+\tau)}\right)\in\{0,1\}^{s-1}$ at any time instant $t_0+\tau\geq t_0$ of the ongoing computation of ${\cal N}$ over $\mathbf{u}$, are uniquely determined by the binary states $\mathbf{y}_\beta^{(t_0+1)}=\tilde{\mathbf{y}}^{(t_0)}=\left(y_1^{(t_0)},\ldots,y_{s-1}^{(t_0)}\right)\in\{0,1\}^{s-1}$ of ${\cal N}$ and $\mathbf{y}_\alpha^{(t_0+1)}\in\{0,1\}^{p+1}$ due to $\mathbf{y}_\alpha^{(t_0+1)}$ is unique for $I_r\ni y_s^{(t_0)}$. In particular, the binary state $y_{\out}^{(\tau_{n+|\mathbf{u}|})}\in\{0,1\}$ of the output neuron $\mbox{out}\in V$ in ${\cal N}$ after the suffix $\mathbf{u}$ has been read, where $t_0<\tau_{n+|\mathbf{u}|}\leq t_0+T$, is uniquely determined by the binary states $\mathbf{y}_\alpha^{(t_0+1)}$ and $\mathbf{y}_\beta^{(t_0+1)}$, according to the output protocol (\ref{oprot}).
 
In other words, there is a Boolean function $f_\mathbf{u}:\{0,1\}^{p+s}\rightarrow\{0,1\}$ such that $f_\mathbf{u}\left(\mathbf{y}_\alpha^{(t_0+1)},\mathbf{y}_\beta^{(t_0+1)}\right)=1$ if{f} $\mathbf{x}\cdot\mathbf{u}\in{\cal L}({\cal N})$ if{f} $\mathbf{x}\in L/\mathbf{u}$.
We define the Boolean function $f:\{0,1\}^{p+s}\rightarrow\{0,1\}$ as the conjunction $f=\neg f_{\mathbf{u}_1}\wedge f_{\mathbf{u}_2\cdot\mathbf{u}_1}$ 
where $\neg$ denotes the negation, or $f=f_{\mathbf{u}_1}\wedge\neg f_{\mathbf{u}_2\cdot\mathbf{u}_1}$ which satisfies $f\left(\mathbf{y}_\alpha^{(t_0+1)},\mathbf{y}_\beta^{(t_0+1)}\right)=1$ if{f} $\mathbf{x}\in L_2\setminus L_1$ or $\mathbf{x}\in L_1\setminus L_2$, respectively. The Boolean function $f$ can be computed by a binary-state two-layered neural network ${\cal N}_f$ that implements e.g.\ the disjunctive normal form of $f$. As depicted in Figure~\ref{quotient}, the network ${\cal N}_f$ is integrated into ${\cal N'}$ so that the neurons $\alpha\cup\beta\subset\tilde{V}'$ create the input layer to ${\cal N}_f$, while the output of ${\cal N}_f$ represents the output neuron $\mbox{out}'\in\tilde{V}'$ of ${\cal N}'$ which thus produces $y_{\out}^{(t_0+3)}=f\left(\mathbf{y}_\alpha^{(t_0+1)},\mathbf{y}_\beta^{(t_0+1)}\right)$. Hence, ${\cal N}'$ recognizes ${\cal L}({\cal N}')=L_2\setminus L_1$ respectively ${\cal L}({\cal N}')=L_1\setminus L_2$ with the delay of 3 computational steps, which completes the proof of Lemma~\ref{symdif}.~\qed 
\end{proof}

\section{Separation of 1ANNs by DCFLs}
\label{separation}

In this section, we will show the main result that any non-regular DCFL cannot be recognized online by a binary-state 1ANN with one extra analog unit, which gives the stronger separation $($DCFLs $\setminus$ REG$)\,\subset\,($2ANNs $\setminus$ 1ANNs$)$ in the analog neuron hierarchy, implying 1ANNs $\cap$ DCFLs $=$ 0ANNs $=$ REG. The class of non-regular DCFLs is thus contained in 2ANNs with rational weights and has the empty intersection with 1ANNs, as depicted in Figure~\ref{analhier}. For the proof, we will exploit the following fact that at least one DCFL cannot be recognized by any 1ANN, which has been shown in our previous work:
\begin{theorem}\cite[Theorem~1]{Sima20}
\label{wsep}
The non-regular deterministic context-free language $L_\#=\{0^n1^n\,|\,n\geq 1\}\subset\{0,1\}^*$ over the binary alphabet cannot be recognized by any 1ANN with one extra analog unit having real weights.
\end{theorem}

In order to generalize Theorem~\ref{wsep} to all non-regular DCFLs, we have shown that $L_\#$ is in some sense the simplest DCFL which is contained in every non-regular DCFL, as is formalized in the following Theorem~\ref{nrdcfl}.
\begin{theorem}\cite[Theorem~1]{JanSim21}
\label{nrdcfl}
Let $L\subseteq\Sigma^*$ be a non-regular deterministic context-free language over an alphabet~$\Sigma$. Then there exist nonempty words $\mathbf{v}_1,\mathbf{v}_2,\mathbf{v}_3,\mathbf{v}_4,$ $\mathbf{v}_5\in\Sigma^+$
and languages $L,L'\in\{L,\overline{L}\}$ such that for every $m\geq 0$,
\begin{equation}
\label{cond}
\mathbf{v}_1\mathbf{v}_2^m\mathbf{v}_3\mathbf{v}_4^n\mathbf{v}_5\,\left\{
\begin{array}{ll}
\notin L&\mbox{for}\enspace 0\leq n<m\\
\in L&\mbox{for}\enspace n=m\\
\in L'&\mbox{for}\enspace n>m\,.
\end{array}
\right. 
\end{equation}
\end{theorem}
This theorem is the basis for the novel concept of so-called DCFL-simple problems, which has been inspired by this study and represents an interesting contribution to the formal language theory. Namely, the DCFL-simple problem $L_\#$ can be reduced to every non-regular DCFL by the truth-table (Turing) reduction using oracle Mealy machines~\cite{JanSim21}. We will show in the following Theorem~\ref{strsep} that this reduction can be implemented by 1ANNs, which generalizes Theorem~\ref{wsep} to any non-regular DCFLs providing the stronger separation of 1ANNs in the analog neuron hierarchy:
\begin{theorem}
\label{strsep}
Any non-regular deterministic context-free language $L\subset\Sigma^*$ over an alphabet $\Sigma$ cannot be recognized online by any 1ANN with one extra analog unit having real weights.
\end{theorem}

\begin{proof}
Let $L\subset\Sigma^*$ be a non-regular deterministic context-free language over an alphabet $\Sigma$ including $q>0$ symbols. 
On the contrary assume that there is a 1ANN ${\cal N}$ that accepts $L={\cal L}({\cal N})$. Let $\mathbf{v}_1,\mathbf{v}_2,\mathbf{v}_3,\mathbf{v}_4,\mathbf{v}_5\in\Sigma^+$ be the nonempty words and $L,L'\in\{L,\overline{L}\}$ be the languages guaranteed by Theorem~\ref{nrdcfl} for $L$, which satisfy condition (\ref{cond}). For any integer constant $c>0$, we can assume without loss of generality that the strings  $\mathbf{v}_i$ have the length at least $c$, that is, 
$|\mathbf{v}_i|\geq c$ for every $i=1,\ldots,5$, since otherwise we can replace $\mathbf{v}_1,\mathbf{v}_2,\mathbf{v}_3,\mathbf{v}_4,\mathbf{v}_5$ by $\mathbf{v}_1\mathbf{v}_2^c,\mathbf{v}_2^c,\mathbf{v}_2^c\mathbf{v}_3\mathbf{v}_4^c,\mathbf{v}_4^c,\mathbf{v}_4^c\mathbf{v}_5$, respectively. According to Lemma~\ref{symdif} for $L_1=L/\mathbf{v}_5$ and $L_2=L/(\mathbf{v}_4\cdot\mathbf{v}_5)$, there is a 1ANN ${\cal N}'$ that accepts ${\cal L}({\cal N}')=L_2\setminus L_1$ if $L=L$, or
${\cal L}({\cal N}')=L_1\setminus L_2$ if $L=\overline{L}$, respectively, with the delay of 3~computational steps. It follows from (\ref{cond}) that for every $m,n\geq 0$,
\begin{equation}
\label{npmn}
\mathbf{v}_1\mathbf{v}_2^m\mathbf{v}_3\mathbf{v}_4^{n-1}\in{\cal L}({\cal N}')\quad\mbox{if{f}}\quad m=n\,,
\end{equation} 
which will be used in the construction of a bigger 1ANN ${\cal N}_\#$ including ${\cal N}'$ as its subnetwork, that recognizes the language $L_\#=\{0^n1^n\,|\,n\geq 1\}$ over the binary alphabet $\{0,1\}$. The architecture of ${\cal N}_\#$ is schematically depicted in Figure~\ref{reduct}. We denote $\tilde{V}'\subset\tilde{V}_\#$ to be the corresponding sets of binary neurons in ${\cal N}'$ and ${\cal N}_\#$, respectively, while ${\cal N}_\#$ shares the only analog unit with ${\cal N}'$.
\begin{figure}[t]
\centering
\includegraphics[width=10.88cm]{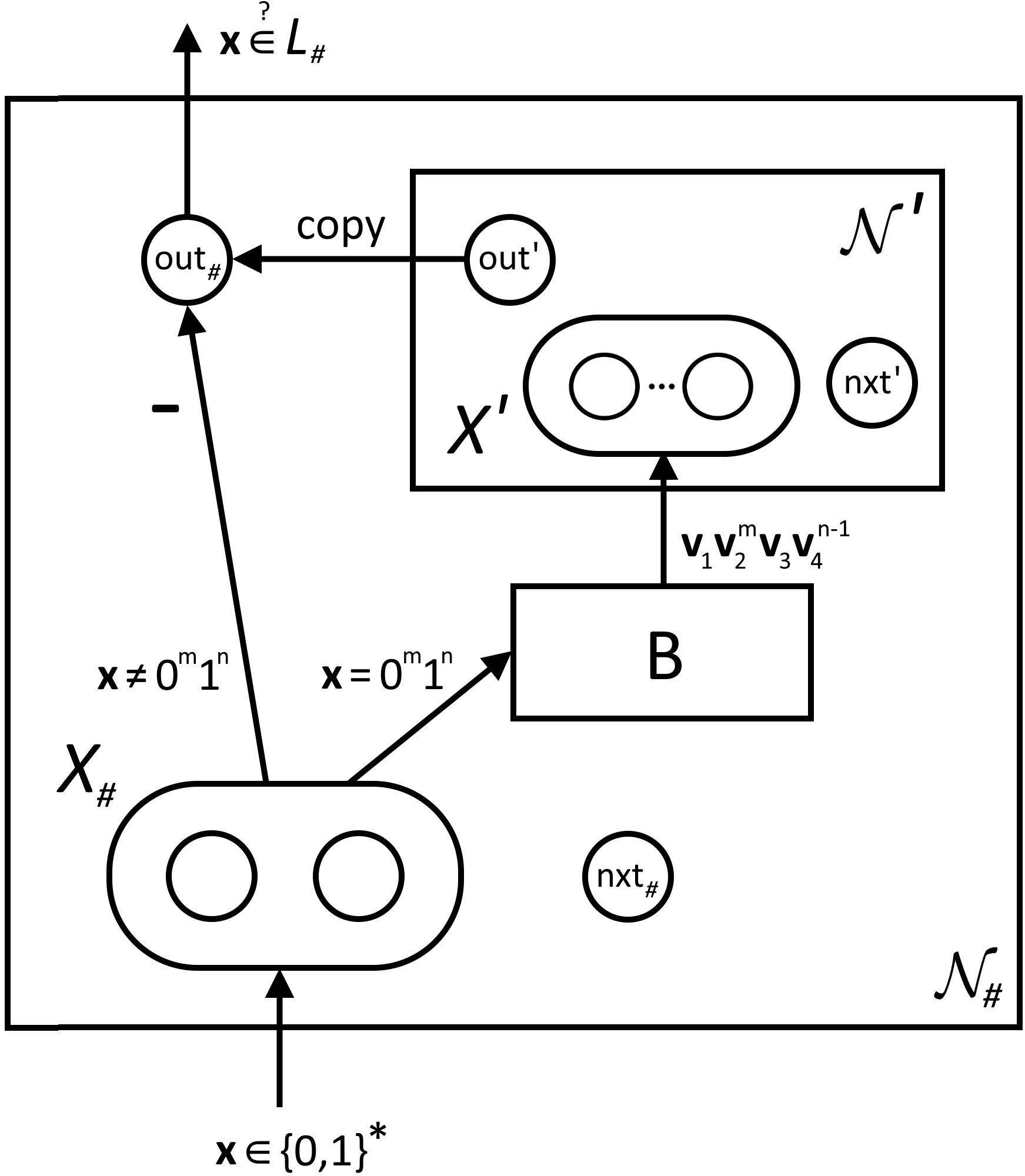}
\caption{The reduction of $L_\#$ to a non-regular DCFL $L$.}
\label{reduct}
\end{figure}

Namely, an input $\mathbf{x}=x_1\ldots x_r\in\{0,1\}^*$ to ${\cal N}_\#$ of the valid form $0^m1^n$ is translated to the string $\mathbf{v}_1\mathbf{v}_2^m\mathbf{v}_3\mathbf{v}_4^{n-1}\in\Sigma^*$ and presented to its subnetwork ${\cal N}'$ which decides online whether $m=n$ according (\ref{npmn}). The result is used by ${\cal N}_\#$ for deciding whether $\mathbf{x}\in L_\#$. For this purpose, ${\cal N}_\#$ contains a finite buffer memory $B$ organized as the queue of current input symbols from $\Sigma^*$, which are presented online, one by one, to ${\cal N}'$ through its $q$ input neurons $X'\subset\tilde{V}'$ by using the one-hot encoding of $\Sigma$, when queried by $\mbox{nxt}'\in\tilde{V}'$ according to the input protocol (\ref{nxt}) and (\ref{iprot}) for ${\cal N}'$. 

At the beginning, $B$ is initialized with the nonempty string $\mathbf{v}_1\in\Sigma^+$ and ${\cal N}_\#$ queries on the first input bit $x_1\in\{0,1\}$, that is, $y_{\nxt_\#}^{(0)}=1$ where $\mbox{nxt}_\#\in\tilde{V}_\#$, according to the input protocol (\ref{nxt}) and (\ref{iprot}) for ${\cal N}_\#$. Thus, at the time instant $\tau_1=1$, ${\cal N}_\#$ reads the first input bit $x_1$ through its two input neurons $X_\#\subset\tilde{V}_\#$ by using the one-hot encoding of $\{0,1\}$. If $x_1=1$, then $\mathbf{x}=1\mathbf{x}'\notin L_\#$ is further rejected for any suffix $\mathbf{x}'\in\{0,1\}^*$ by clamping the state $y_{\out_\#}^{(t)}=0$ of the output neuron $\mbox{out}_\#\in\tilde{V}_\#$ in ${\cal N}_\#$ whereas  $y_{\nxt_\#}^{(t)}=1$, for every $t>1$. If $x_1=0$, then ${\cal N}_\#$ writes the string $\mathbf{v}_2\in\Sigma^+$ to $B$. At the same time, the computation of ${\cal N}'$ proceeds while reading its input from the buffer $B$ when needed which is indicated by the neuron $\mbox{nxt}'\in\tilde{V}'$ one computational step beforehand. Every time before $B$ becomes empty, ${\cal N}_\#$ reads the next input bit $x_k\in\{0,1\}$ for $k>1$ and writes the string $\mathbf{v}_2\in\Sigma^+$ to $B$ if $x_k=0$, so that ${\cal N}'$ can smoothly continue in its computation. This is repeated until ${\cal N}_\#$ reads the input bit $x_{m+1}=1$ for $m\geq 1$, which completes the first phase of the computation by ${\cal N}_\#$. In the course of this first phase, each prefix $0^k\notin L_\#$ of the input word $\mathbf{x}$, which is being read online by ${\cal N}_\#$, is rejected by putting the state $y_{\out_\#}^{(\tau_{k+1})}=0$ of its output neuron $\mbox{out}_\#$ for every $k=1,\ldots,m$, according to the output protocol (\ref{oprot}) for ${\cal N}_\#$.

At the beginning of the subsequent second phase when the input bit $x_{m+1}=1$ has been read, ${\cal N}_\#$ writes the string $\mathbf{v}_3\mathbf{v}_4\in\Sigma^+$ to $B$ and continues uninterruptedly in the computation of ${\cal N}'$ over the input being read from the buffer $B$ when required. Every time before $B$ becomes empty which will precisely be specified below, ${\cal N}_\#$ reads the next input bit $x_{m+n}\in\{0,1\}$ for $n>1$ and writes the string $\mathbf{v}_4\in\Sigma^+$ to $B$ if $x_{m+n}=1$, so that ${\cal N}'$ can smoothly carry out its computation. If $x_{m+n}=0$, then $\mathbf{x}=0^m1^{n-1}0\,\mathbf{x}'\notin L_\#$ is further rejected for any suffix $\mathbf{x}'\in\{0,1\}^*$ by clamping the states $y_{\out_\#}^{(t)}=0$ and $y_{\nxt_\#}^{(t)}=1$ since that. 

It follows that in the second phase, ${\cal N}'$ decides online for each $n>0$ whether the input word $\mathbf{v}_1\mathbf{v}_2^m\mathbf{v}_3\mathbf{v}_4^{n-1}\in\Sigma^+$ of length $\ell=|\mathbf{v}_1\mathbf{v}_3|+m\cdot|\mathbf{v}_2|+(n-1)\cdot|\mathbf{v}_4|$ belongs to ${\cal L}({\cal N}')$, where the result is indicated through its output neuron $\mbox{out}'\in\tilde{V}'$ at the time instant $\tau_{\ell+1}'+3$ with the delay of 3 computational steps after the next symbol subsequent to $\mathbf{v}_1\mathbf{v}_2^m\mathbf{v}_3\mathbf{v}_4^{n-1}$ is read, according to the delayed output protocol (\ref{oprot}) for ${\cal N}'$. For sufficiently large length $|\mathbf{v}_4|>3$, the output neuron $\mbox{out}'$ thus signals whether $\mathbf{v}_1\mathbf{v}_2^m\mathbf{v}_3\mathbf{v}_4^{n-1}\in{\cal L}({\cal N}')$, while still reading the next string $\mathbf{v}_4$ corresponding to the last input bit $x_{m+n}=1$ of the current input $0^m1^n$ to ${\cal N}_\#$. At the next time instant $\tau_{m+n+1}=\tau_{\ell+1}'+4$, when the subsequent input bit $x_{m+n+1}\in\{0,1\}$ is presented to ${\cal N}_\#$, which is queried by ${\cal N}_\#$ via the state $y_{\nxt_\#}^{(\tau_{m+n+1}-1)}=1$ of the neuron $\mbox{nxt}_\#$ one step beforehand, the output neuron $\mbox{out}_\#$ of ${\cal N}_\#$ copies the state of $\mbox{out}'$, providing the result of the computation by ${\cal N}_\#$ over the input word $\mathbf{x}\in\{0,1\}^*$ according to the output protocol (\ref{oprot}) for~${\cal N}_\#$. Namely, $y_{\out_\#}^{(\tau_{m+n+1})}=1$ if{f} $\mathbf{v}_1\mathbf{v}_2^m\mathbf{v}_3\mathbf{v}_4^{n-1}\in{\cal L}({\cal N}')$ if{f} $m=n$ if{f} $0^m1^n\in L_\#$ according to (\ref{npmn}), which ensures ${\cal L}({\cal N}_\#)=L_\#$.

The preceding online reduction of any input $0^m1^n$ for ${\cal N}_\#$ to the input $\mathbf{v}_1\mathbf{v}_2^m\mathbf{v}_3\mathbf{v}_4^{n-1}$ for ${\cal N}'$ can clearly be realized by a finite automaton, including the implementation of the finite buffer memory $B$. This finite automaton can further be implemented by a~binary-state neural network by using the standard constructions~\cite{Horne96,Indyk95,Minsky67,Sima98}, which is wired to the 1ANN ${\cal N}'$ in order to create the 1ANN ${\cal N}_\#$ recognizing the language ${\cal L}({\cal N}_\#)=L_\#$ online, as described above. In particular, the synchronization of these two networks is controlled by their input/output protocols, while the operation of ${\cal N}'$ can suitably be slowed down for sufficiently large length of strings $\mathbf{v}_i$. However, we know by Theorem~\ref{wsep} that there is no 1ANN that accepts $L_\#$, which is a contradiction completing the proof of Theorem~\ref{strsep}.~\qed 
\end{proof}

\section{Conclusion}
\label{concl}

In this paper, we have refined the analysis of the computational power of discrete-time binary-state recurrent neural networks $\alpha$ANNs extended with $\alpha$ analog-state neurons by proving a stronger separation 1ANNs $\subsetneqq$ 2ANNs in the analog neuron hierarchy depicted in Figure~\ref{analhier}. Namely, we have shown that the class of non-regular DCFLs is contained in 2ANNs $\setminus$ 1ANNs, which implies 1ANNs $\cap$ DCFLs $=$ 0ANNs $=$ REG. For this purpose, we have reduced the non-regular DCFL $L_\#=\{0^n1^n\,|\,n\geq 1\}$, which is known to be not in 1ANNs~\cite{Sima20}, to any non-regular DCFL.

It follows that $L_\#$ is in some sense the simplest languages in the class of non-regular DCFLs. This is by itself an interesting contribution to computability theory, which has inspired the novel concept of a DCFL-simple problem that can be reduced to any non-regular DCFL by the truth-table (Turing) reduction using oracle Mealy machines~\cite{JanSim21}. The proof of the stronger separation 1ANNs $\subsetneqq$ 2ANNs thus represents the first non-trivial application of this concept. We believe that this approach can open a new direction of research aiming towards the existence of the simplest problems in traditional complexity classes as a methodological counterpart to the hardest problems in a class (such as NP-complete problems in NP) to which all the problems in this class are reduced. We conjecture that our separation result can further be strengthen to \emph{nondeterministic} context-free languages (CFLs) by showing that 1ANNs $\cap$ CFLs $=$ 0ANNs.

Moreover, it is an open question whether there is a non-context-sensitive language that can be accepted \emph{offline} by a 1ANN, which does not apply to an online input/output protocol since we know online 1ANNs $\subset$ CSLs. Another important challenge for future research is the separation 2ANNs $\subsetneqq$ 3ANNs of the second level in the analog neuron hierarchy and the relation between 2ANNs and CFLs, e.g.\ the issue of whether \mbox{2ANNs $\cap$ CFLs $\stackrel{?}{=}$ DCFLs.}

It also appears that the analog neuron hierarchy is only partially comparable to that of Chomsky since 1ANNs and probably also 2ANNs do not coincide with the Chomsky levels although 0ANNs and 3ANNs correspond to FAs and TMs, respectively. In our previous paper~\cite{Sima19}, the class of languages accepted by 1ANNs has been characterized syntactically by so-called cut languages which represent a new type of basis languages defined by NNs that do not have an equivalent in the Chomsky hierarchy. A~similar characterization still needs to be done for 2ANNs. 

The analog neuron hierarchy shows what is the role of analogicity in the computational power of NNs. The binary states restrict NNs to a finite domain while the analog values create a potentially infinite state space which can be exploited for recognizing more complex languages in the Chomsky hierarchy. This is not only an issue of increasing precision of rational-number parameters in NNs but also of functional limitations of one or two analog units for decoding an information from rational states as well as for synchronizing the storage operations. An important open problem thus concerns the generalization of the hierarchy to other types of analog neurons used in practical deep networks such as LSTM, GRU, or ReLU units~\cite{Korsky19,Merrill20}. Clearly, the degree of analogicity represent another computational resource that can simply be measured by the number of analog units while a possible tradeoff with computational time can also be explored.

Nevertheless, the ultimate goal is to prove a proper ``natural'' hierarchy of NNs between integer and rational weights similarly as it is known between rational and real weights~\cite{Balcazar97} and possibly, map it to known hierarchies of regular/con\-text-free languages. This problem is related to a more general issue of finding suitable complexity measures of realistic NNs establishing the complexity hierarchies, which could be employed in practical neurocomputing, e.g.\ the precision of weight parameters~\cite{Weiss18}, energy complexity~\citep{Sima14}, temporal coding etc.

Yet another important issue concerns grammatical inference. For a given PDA or TM, the constructions of computationally equivalent 2ANNs and 3ANNs, respectively, can be implemented algorithmically~\cite{Sima20} although they do not provide learning algorithms that would infer a language from training data. Nevertheless, the underlying results establish the principal limits (lower and upper bounds) for a few analog units to recognize more complex languages. For example, we now know that one analog neuron cannot accept even some simple DCFLs. In other words, any learning algorithm has to employ a sufficient number of analog units to be able to infer more complex grammars.

\section*{Acknowledgments}
The research was done with institutional support RVO: 67985807 and 
\linebreak
partially supported by the grant of the Czech Science Foundation No. 
\linebreak
\mbox{GA19-05704S.}

\bibliographystyle{elsarticle-num}
\bibliography{sima}

\end{document}